\documentclass[11pt]{article}

\usepackage[total={6.5in, 9.5in}]{geometry}
\usepackage{xcolor}
\usepackage{amsmath}
\usepackage{amssymb}
\usepackage{graphicx}
\usepackage{amsthm}
\usepackage{hyperref}
\usepackage{natbib}

\usepackage[T1]{fontenc} \usepackage[utf8]{inputenc}

\usepackage[labelfont=bf]{caption}
\usepackage[format=hang]{subcaption}

\usepackage{tikz}
\usetikzlibrary{bayesnet}
\usepackage{xfrac}
\usepackage{cleveref}
\usepackage{bm}
\newtheorem{assumption}{Assumption}
\newtheorem{proposition}{Proposition}

\newtheorem{definition}{Definition}

\newcommand{\f}{\mathrm{f}}
\newcommand{\F}{\mathrm{F}}
\newcommand{\g}{\mathrm{g}}

\newcommand{\pr}{\mathrm{p}}

\newcommand{\Prob}{\mathrm{P}}
\newcommand{\pa}{\mathrm{pa}}
\newcommand{\rmdo}{\mathrm{do}}
\newcommand{\di}{\mathrm{d}}
\newcommand{\xb}{\mathbf{x}}

\newcommand{\yb}{\mathbf{y}}
\newcommand{\ab}{\mathbf{a}}
\newcommand{\zb}{\mathbf{z}}
\newcommand{\ub}{\mathbf{u}}
\newcommand{\epsb}{\bm{\epsilon}}
\newcommand{\Xc}{\mathcal{X}}
\newcommand{\Yc}{\mathcal{Y}}
\newcommand{\Gb}{\mathbb{G}}
\newcommand{\Ib}{\mathbb{I}}
\newcommand{\Nb}{\mathbb{N}}
\newcommand{\Sb}{\mathbb{S}}
\newcommand{\Tb}{\mathbb{T}}
\newcommand{\Rb}{\mathbb{R}}
\newcommand{\Zb}{\mathbb{Z}}
\newcommand{\Eb}{\mathbb{E}}

\newcommand{\Ac}{\mathcal{A}}
\newcommand{\Ec}{\mathcal{E}}

\newcommand{\Nc}{\mathcal{N}}
\newcommand{\Mc}{\mathcal{M}}
\newcommand{\Gc}{\mathcal{G}}
\newcommand{\Vc}{\mathcal{V}}
\newcommand{\Cc}{\mathcal{C}}

\newcommand{\Dc}{\mathcal{D}}
\newcommand{\Pc}{\mathcal{P}}

\newcommand{\Zc}{\mathcal{Z}}
\newcommand{\GPc}{\textsc{gp}}\newcommand{\ID}{\mathrm{Id}}
\newcommand{\s}{;}
\newcommand{\obs}{\mathrm{obs}}
\newcommand{\scm}{\mathrm{scm}}

\newcommand{\Cov}{\mathrm{Cov}}
\newcommand{\Ob}{\mathbb{O}}
\newcommand{\Zero}{\mathbf{0}}
\newcommand{\Yb}{\mathbf{Y}}
\newcommand{\wt}{\textsc{ref}}

\newcommand{\Oarray}{(\Nb \cup \diamond)^2}
\newcommand{\Sequence}{\textsc{Sequence Data}}
\newcommand{\Spatial}{\textsc{Spatial Data}}
\newcommand{\Array}{\textsc{Array Data}}
\newcommand{\ball}{{\mathcal{N}(\omega)}}

\usepackage{algpseudocode}
\usepackage{algorithm}

\usepackage{enumitem}

\usepackage{tikz}
\usetikzlibrary{arrows.meta,calc}
\usepackage{amsmath,bm}
\usepackage{graphicx}
\definecolor{nodefill}{HTML}{9BA2EA}
\definecolor{nodering}{HTML}{B9BEFB}
\definecolor{uring}{HTML}{D2D2D2}
\definecolor{arrowgray}{HTML}{C7C7C7}
\definecolor{callgray}{HTML}{5F5F5F}
\definecolor{hifill}{HTML}{7E87E6}
\definecolor{histroke}{HTML}{8A8A8A}
\definecolor{med}{HTML}{A9B0F2}
\definecolor{light}{HTML}{C3C8F5}
\definecolor{gridline}{HTML}{ECE7DA}
\definecolor{axisgray}{HTML}{7A7A7A}
\tikzset{
  var/.style={circle, draw=nodering, line width=1.1pt, fill=nodefill,
              minimum size=1.05cm, inner sep=1pt, font=\large},
  exo/.style={var, fill=white, draw=uring, line width=0.9pt, minimum size=0.95cm},
  cause/.style={arrowgray, line width=0.9pt, -{Stealth[length=7pt, width=6pt]}},
  callout/.style={callgray, line width=0.55pt},
  axislab/.style={axisgray, font=\large\itshape},
}
 
\begin{document}

\title{Geometric Causal Models}

\author{Eli N. Weinstein\footnote{Department of Chemistry, Technical University of Denmark, Kgs. Lyngby, DK. \url{enawe@dtu.dk}}, David M. Blei\footnote{Departments of Statistics and Computer Science, Columbia University, New York, NY, USA. \url{david.blei@columbia.edu}}}

\date{\today}

\maketitle

\begin{abstract}
Scientists often seek to draw causal inferences from structured data that is not independently and identically distributed, such as spatial data, network data, or molecular data. 
We develop \textit{geometric causal models} (GCMs), a framework for causal inference from dependent data that exploits underlying symmetries of the data generating process. 
For example, in spatial data, we consider processes that are symmetric under translations, or in graph data, symmetric under permutations of the nodes.
We show how symmetries, formalized with group theory, can enable causal identification and estimation.
We deploy ergodic theory for amenable groups to establish identification, and combine geometric deep learning with scalable Bayesian inference for estimation. We recover i.i.d. causal models and do-calculus when the data is a sequence and the symmetry is permutation equivariance, and find novel types of causal models when we use alternate structures and symmetries.
As an example, we construct a causal model that satisfies the symmetries of DNA.
This GCM enables new estimators for the effects of genetic variation, combining deep functional genomics models to describe outcomes and DNA language models to describe propensities. We illustrate on semisynthetic data.

\end{abstract}

\section{Introduction}

In classical causal inference, each unit exists in isolation: its treatment and outcome is independent of all other units.
But consider a dataset about trees and temperature measured at different locations in a city. 
We expect spatial dependence: planting trees will affect temperature in the surrounding area.
Similarly in a community: talking to one person can affect the decisions of their friends.
Or in a genome: a mutation at one locus affects gene expression nearby along the chromosome.

A central challenge in causal inference is to develop methods to account for such dependence among units. Many methods have been proposed, including techniques tailored to spatial, temporal or network data, as well as techniques for handling less structured forms of interference and dependence~\citep{Giffin2021-qd,Giffin2023-qm,Papadogeorgou2019-ij,Papadogeorgou2023-mp,Gilbert2024-oc,Wang2025-in,Rischard2018-nw,Peters2013-to,Bojinov2019-nw,Christiansen2022-ip,Papadogeorgou2022-pp,Ogburn2022-iv,Sridhar2022-tu,Cristali2022-cv,Hudgens2008-hn,Guo2025-rd,Agarwal2021-ds,Savje2021-zy,Liang2023-qz,Roysland2024-mq,Didelez2008-fk}.

We study non-i.i.d. causal inference through the unified lens of symmetry.
Our basic idea is that even when units are not independent, their underlying causal mechanisms may still be \emph{symmetric}, in the sense that they are preserved under certain transformations. 
For instance, while temperature might not be i.i.d. across locations in a city, the same physical forces are at play everywhere. So the causal mechanisms driving urban heat are translation invariant, symmetric under a shift in location.

We formalize the problem using the tools of group theory.
We propose \textit{geometric causal models} (GCMs), causal models with mechanisms that are symmetric under a group of transformations.
For spatial data, this can be the group of translations in space; for network data, this can be the group of permutations of the network's nodes.
We show that as long as causal mechanisms obey some symmetry, causal inference is possible, even in the face of complex dependence among units.

To create GCMs, we generalize the axioms of structural causal modeling~\citep{Bareinboim2021-ro,Pearl2009-fh}.
We introduce causal mechanisms that are \textit{equivariant} to a symmetry group.
Then we analyze causal identification in these GCMs, to determine when and how we can estimate causal mechanisms. 
We consider arbitrary causal graphs and nonparametric causal mechanisms.
Relying on tools from ergodic theory, we show that causal inference is possible even with data from a single city, social network, or genome.
We also derive causal estimators. 
We show how statistical and machine learning methods for handling structured dependent data, such as those employed in Bayesian nonparametrics and in geometric deep learning, can be adapted to causal estimation.

We demonstrate this idea with an application to genomics, introducing a causal model to describe how genetic variation affects biological activity along the genome.
The GCM provides a causal understanding of large scale scientific deep learning systems, specifically functional genomics models such as AlphaGenome, and generative models of genomes such as Evo2 \citep{Avsec2026-tq,Brixi2026-sp}.
The GCM also lets us derive alternative estimators for variant effects, which we demonstrate on semisynthetic human functional genomics data.

GCMs unify many causal methods: i.i.d. models, spatial models, network models and more all emerge from the same framework, each corresponding to a different choice of symmetry group.
More importantly, by generalizing causal inference to new symmetry groups, GCMs enable causal inference with new types of structured data.  
They especially expand the reach of causal inference in the natural sciences, beyond the classical territory of biomedical and social science \citep{Imbens2015-bz}. GCMs equip scientists with a toolbox for formalizing causal questions for non-i.i.d. data; correcting for different sources of confounding and error using different causal graphs; and leveraging flexible machine learning methods for estimation.

\subsection{Related work}

\paragraph{Causal inference with non-i.i.d. data.} GCMs connect to previous methods for causal inference with dependent and non-i.i.d. data, including methods for 
spatial data~\citep[e.g.][]{Giffin2021-qd,Giffin2023-qm,Papadogeorgou2019-ij,Papadogeorgou2023-mp,Gilbert2024-oc,Wang2025-in,Rischard2018-nw}, 
time series data~\citep[e.g.][]{Peters2013-to,Bojinov2019-nw,Liang2023-qz}, spatiotemporal data~\citep[e.g.][]{Christiansen2022-ip,Papadogeorgou2022-pp}, 
network data~\citep[e.g.][]{Ogburn2022-iv,Sridhar2022-tu,Cristali2022-cv,Li2022-iw},
array/panel data~\citep[e.g.][]{Wooldridge2010-dv,Agarwal2021-ds,Athey2021-zu}, and
clustered/hierarchical data~\citep[e.g.][]{Hudgens2008-hn,Liu2014-fl,Peters2016-qg,Guo2025-rd,Feller2015-eb,Weinstein2026-jl}.

GCMs offer a unified perspective on these different models, through the lens of group theory: spatiotemporal models are causal models that are symmetric under shifts in space or time; network models are symmetric under permutations of the nodes; array models are symmetric under separate permutations of the rows and columns; cluster models are symmetric under permutations of the clusters, and under permutations of the units within each cluster.
Though many previous methods focus on a specific causal graph, GCMs cover a wide diversity~\citep[][]{Christiansen2022-ip,Weinstein2026-jl,Guo2025-rd}.

\citet{Dance2024-mi} also apply group theory to causal inference, but in the setting of i.i.d. models, and for the purpose of efficient estimation.
\citet{Besserve2018-rp} use group theory to infer causal direction from i.i.d. data, evaluating independence of cause and mechanism.

\paragraph{Symmetry in statistics and machine learning.} Symmetry, formalized with group theory, has a long history in statistics~\citep{Eaton2021-ns,Diaconis1988-gy}. We build most directly on \citet{Orbanz2015-fm}, who unify different nonparametric Bayesian models of structured data through the lens of invariant and ergodic distributions. \citet{Austern2018-dy} show how these symmetries enable statistical inference, deriving central limit theorems for invariant distributions.
We employ these ideas to show how symmetry enables causal inference.
This relates to other recent efforts to incorporate symmetry into inference~\citep{Dobriban2025-io,Chiu2024-og}.

We also build on geometric deep learning, which constructs flexible models to make predictions about structured data~\citep{Bronstein2021-gz,Bruna2014-ux}.
Geometric deep learning provides tools to build neural network architectures that respect different group symmetries~\citep{Cohen2016-oy,Kondor2018-pz}.
This approach has seen wide and growing application especially in the natural sciences, as many physical and chemical laws and theories respect some nontrivial symmetry~\citep{Otto2023-dk,Bogatskiy2022-oj,Watson2023-sp}.
Geometric deep learning is often motivated by the aim of more efficient prediction;
our results go further, showing symmetry enables causal identification. 

GCMs are probabilistic, and so interface especially with probabilistic approaches to geometric deep learning~\citep{Bloem-Reddy2020-gm}.
In particular, a powerful strategy for constructing invariant generative models is to compose i.i.d. noise with equivariant functions~\citep{Rezende2019-yw,Kohler2020-sx,Klein2023-iq,Midgley2023-mg}.
GCMs apply this construction to structural causal modeling.

\paragraph{Note on nomenclature.} We use the term \textit{invariant} in the group theoretic sense, to describe an object that is unchanged under the action of a group.
This is distinct from common usage in causal machine learning, where it describes a distribution or model that is unchanged across different environments~\citep{Peters2016-qg,Arjovsky2019-bx}.

\section{A first look at geometric causal models}

\begin{figure}[t]
\centering
\begin{subfigure}[t]{0.3\textwidth}
\centering
\scalebox{0.6}{\begin{tikzpicture}
\node[exo] (u) at (0,0)      {$u_\omega$};
  \node[var] (x) at (2.8,0)    {$x_\omega$};
  \node[var] (a) at (0,-2.2)   {$a_\omega$};
  \node[var] (y) at (2.8,-2.2) {$y_\omega$};
  \draw[cause] (u) -- (x);
  \draw[cause] (u) -- (a);
  \draw[cause] (x) -- (y);
  \draw[cause] (a) -- (y);
\foreach \i in {-2,-1,1,2}
    \draw[rounded corners=4.7pt, line width=1pt, draw=nodering, fill=med]
      ({1.4+\i*1.45-0.485},-3.9) rectangle ({1.4+\i*1.45+0.485},-4.87);
  \draw[rounded corners=4.7pt, line width=0.55pt, draw=histroke, fill=hifill]
      (0.915,-3.9) rectangle (1.885,-4.87);
\coordinate (AL) at (1.08,-4.065);
  \coordinate (AR) at (1.72,-4.065);
  \draw[callout] (a) -- ($(AL)!0.165cm!(a.center)$);
  \draw[callout] (y) -- ($(AR)!0.165cm!(y.center)$);
  \node at (1.4,-5.22) {\large$\bm{\omega}$};
  \node[axislab, anchor=east] at (4.88,-5.35) {units};
\end{tikzpicture}}
\caption*{(a) \Sequence}
\end{subfigure}
\begin{subfigure}[t]{0.3\textwidth}
\centering
\scalebox{0.45}{\begin{tikzpicture}
\node[exo] (u) at (0,0)      {$u_\omega$};
  \node[var] (x) at (3,0)      {$x_\omega$};
  \node[var] (a) at (0,-2.3)   {$a_\omega$};
  \node[var] (y) at (3,-2.3)   {$y_\omega$};
  \draw[cause] (u) -- (x);
  \draw[cause] (u) -- (a);
  \draw[cause] (x) -- (y);
  \draw[cause] (a) -- (y);
\fill[med] (-2.3,-3.35) -- (5.3,-3.35) -- (3.593,-6.55) -- (-4.007,-6.55) -- cycle;
\draw[fill=hifill, draw=histroke, line width=0.55pt, line join=round]
    (1.2,-4.55) -- (1.8,-4.55) -- (1.5867,-4.95) -- (0.9867,-4.95) -- cycle;
  \draw[callout] (a) -- (1.2,-4.55);
  \draw[callout] (y) -- (1.8,-4.55);
  \node at (1.328,-5.14) {\large$\bm{\omega}$};
  \node[axislab] at (-0.2,-7.15) {longitude};
\node[axislab, rotate=61.9] at (4.93,-5.21) {latitude};
\end{tikzpicture}}
\caption*{(c) \Spatial}
\end{subfigure}
\begin{subfigure}[t]{0.3\textwidth}
\centering
\scalebox{0.5}{\begin{tikzpicture}
\node[exo] (u) at (0,0)        {$u_\omega$};
  \node[var] (a) at (-1.9,-2.05) {$a_\omega$};
  \node[var] (x) at (0,-2.05)    {$x_\omega$};
  \node[var] (y) at (1.9,-2.05)  {$y_\omega$};
  \draw[cause] (u) -- (a);
  \draw[cause] (u) -- (y);
  \draw[cause] (a) -- (x);
  \draw[cause] (x) -- (y);
\foreach \i in {0,...,3}
    \draw[gridline, line width=0.4pt, fill=light]
      (-4.625,{-3.89-\i*0.69}) rectangle (-2.775,{-3.89-(\i+1)*0.69});
\foreach \j in {0,1,2}
    \draw[gridline, line width=0.4pt, fill=light]
      ({-2.775+\j*1.85},-3.2) rectangle ({-2.775+(\j+1)*1.85},-3.89);
\foreach \i in {0,...,3} \foreach \j in {0,1,2}
    \draw[gridline, line width=0.4pt, fill=med]
      ({-2.775+\j*1.85},{-3.89-\i*0.69}) rectangle ({-2.775+(\j+1)*1.85},{-3.89-(\i+1)*0.69});
\draw[fill=hifill, draw=histroke, line width=0.55pt]
    (-0.6,-4.025) rectangle (0.6,-4.445);
  \draw[callout] (a) -- (-0.6,-4.025);
  \draw[callout] (y) -- (0.6,-4.025);
  \node at (0,-4.75) {\large$\bm{\omega}$};
\draw[callgray, line width=0.7pt]
    (-3.7,-3.37) -- ++(0.18,-0.18) -- ++(-0.18,-0.18) -- ++(-0.18,0.18) -- cycle;
\node[axislab, rotate=90] at (3.2,-4.93) {rows};
  \node[axislab] at (-0.925,-7.15) {columns};
\end{tikzpicture}}
\caption*{(e) \Array}
\end{subfigure}
\begin{subfigure}[t]{0.3\textwidth}
\centering
\begin{tikzpicture}

\node[obs]                               (y) {$\yb$};
  \node[obs, left=1cm of y] (a) {$\ab$};
  \node[latent, above=.6cm of a] (u) {$\ub$};
  \node[obs, above=.6cm of y] (x) {$\xb$};

\edge {x,a} {y} ;
  \edge {u}{a,x};

\plate{in} {(y)(a)(x)} {$\Omega=\Nb,\Gb=\Sb$};

\end{tikzpicture}
\caption*{(b) \textsc{Sequence GCM}} \label{fig:seq_confounder}
\end{subfigure}
\begin{subfigure}[t]{0.3\textwidth}
\centering
\begin{tikzpicture}

\node[obs]                               (y) {$\yb$};
  \node[obs, left=1cm of y] (a) {$\ab$};
  \node[latent, above=.6cm of a] (u) {$\ub$};
  \node[obs, above=.6cm of y] (x) {$\xb$};

\edge {x,a} {y} ;
  \edge {u}{a,x};

\plate{in} {(y)(a)(x)(u)} {$\Omega=\Rb^2,\Gb=\Tb^2$};

\end{tikzpicture}
\caption*{(d) \textsc{Spatial GCM}} \label{fig:space_confounder}
\end{subfigure}
\begin{subfigure}[t]{0.3\textwidth}
\centering
\begin{tikzpicture}

\node[obs]                               (y) {$\yb$};
  \node[obs, left=1.5cm of y] (a) {$\ab$};
  \node[obs, left=.4cm of y] (x) {$\xb$};
  \node[latent, above=.5cm of x] (u) {$\ub$};

\edge {x,u} {y} ;
  \edge {u}{a};
  \edge{a}{x};

\plate{in} {(y)(a)(x)(u)} {$\Omega=\Oarray,\Gb=\Sb^2$};

\end{tikzpicture}
\caption*{(f) \textsc{Array GCM}} \label{fig:array_mediator}
\end{subfigure}
\caption{\textbf{Examples of geometric causal models} (a) Conventional i.i.d. causal models describe data from an unordered sequence of units, $\omega \in \{1, 2, \ldots\}$. (b) In GCM notation, we use structured variables $\xb$ that group together measurements from each unit, and specify the index set $\Omega = \Nb$ and the symmetry group $\Gb = \Sb$, permutations. (c) In spatial data, we measure variables at different points in the plane. (d) The GCM specifies the index $\Omega = \Rb^2$ and the translation group $\Tb^2$. (e) In array data, we measure data at different rows and columns. (f) The GCM specifies the index set $\Omega = \Oarray$, where entries $(\omega_1, \diamond)$ and $(\diamond, \omega_2)$ are the row- or column-specific variables. The group is separate permutations of the rows and columns, $\Gb = \Sb^2$.} \label{fig:examples}
\end{figure}
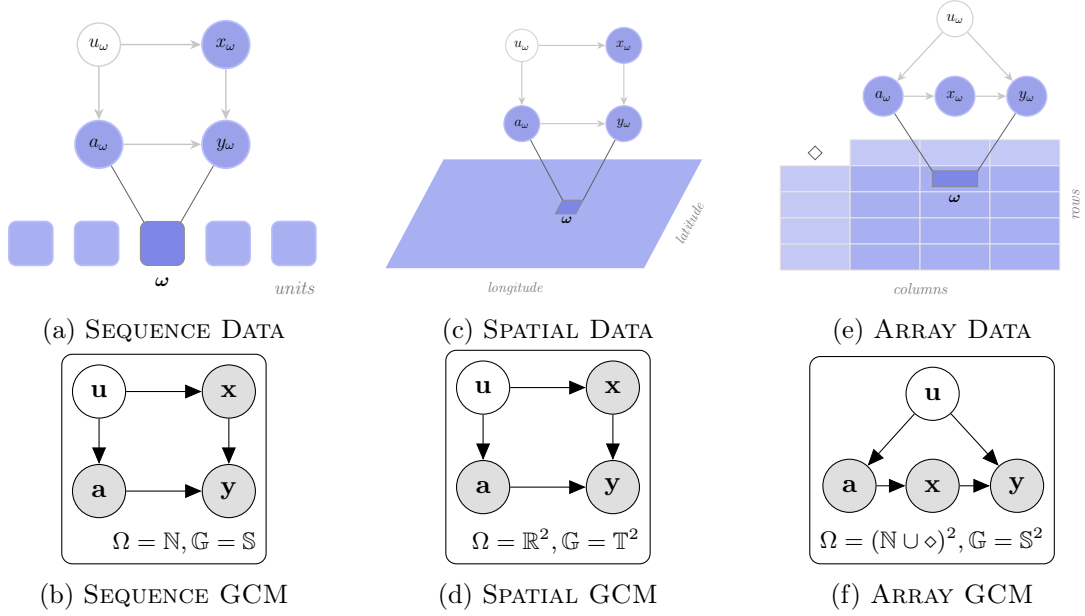

\paragraph{Three examples of geometric causal models.} Causal models describe the world in terms of variables and their impact on one another, mapped out in a graph.
A geometric causal model (GCM) describes \textit{structured} variables $\xb$ whose entries $x_\omega$ are indexed by elements of some set, $\omega \in \Omega$.
For example, these indices could refer to a position in space, or a row and column in an array.
GCMs assume the causal mechanisms underlying structured variables have symmetry: they remain the same even after transformations. The set of transformations they are stable to, $\Gb$, is called a \textit{symmetry group}.
Graphically, in a diagram of a GCM, we notate the index $\Omega$ and the group $\Gb$ in the corner of the plate (\Cref{fig:examples}).

\textit{Sequence data.} \Cref{fig:examples}b is a classical causal model, rewritten as a GCM. 
It describes variables collected from a sequence of units (\Cref{fig:examples}a). 
The index set $\Omega$ is the set of natural numbers, $\Nb = \{1, 2, 3, \ldots\}$. Each variable in the model is a sequence, e.g. $\xb = (x_\omega)_{\omega \in \Nb} = (x_1, x_2, \ldots)$.
There is an unobserved confounder $\ub$, a covariate $\xb$, a treatment $\ab$ and an outcome $\yb$, each a sequence with a value for every unit $\omega \in \Nb$.
The underlying causal mechanisms of the model are assumed to be symmetric with respect to permutations, i.e. the ordering of the units does not affect how the causal variables are generated. 
The group of permutations is denoted $\Sb$.

\textit{Spatial data.} \Cref{fig:examples}d is a GCM for spatial data. The causal graph is the same as in \Cref{fig:examples}b. But the variables are now spatial processes: the index set is the Euclidean plane $\Omega = \Rb^2$ (\Cref{fig:examples}c). Each $x_\omega$ is the value of $\xb$ at latitude $\omega_1$ and longitude $\omega_2$.
 For example, we may be interested in the impact of trees on urban temperature. $\ab$ could describe number of trees at different locations, and $\yb$ the average temperature. Covariates $\xb$ could include other observables about the environment, such as street traffic or building materials, while hidden confounders $\ub$ could account for local business or zoning.
 
There is interference across space: a park full of trees at $a_\omega$ can impact temperature a block away $y_{\omega'}$.
But we assume the causal mechanisms creating this interference, i.e. the underlying forces of physics and microclimate, are unchanged across locations.
So, if we observe trees in Brooklyn we can make predictions about trees in Manhattan.
This GCM is symmetric with respect to shifts in location. The group of translations of the Euclidean plane is denoted $\Tb^2$.

\textit{Array data.} \Cref{fig:examples}f is a GCM for array data. The index $\omega=(\omega_1, \omega_2) \in \Nb^2$ specifies the row $\omega_1$ and column $\omega_2$ (\Cref{fig:examples}e). 
Each $x_\omega$ is an entry in the array $\xb$.
For example, we might be interested in the impact of large language model (LLM) tutoring on student outcomes.
The treatment $a_{\omega_1 \omega_2}$ is whether student $\omega_1$ asks the LLM about topic $\omega_2$, the mediator $x_{\omega_1 \omega_2}$ is the explanation generated by the LLM, and the outcome $y_{\omega_1 \omega_2}$ is how well the student scores on a question about the topic. 

The causal model should account for individual students and topics.
For example, student's questions and outcomes may be confounded by their interests and beliefs, or a topic's inherent difficulty. 
We use $\omega_1 \diamond$ and $\diamond \omega_2$ to index row- and column-specific variables, student-specific or topic-specific.
The confounder $\ub$ includes student-specific values $u_{\omega_1 \diamond}$, topic-specific values $u_{\diamond \omega_2}$, and student-topic-specific values $u_{\omega_1 \omega_2}$.

The GCM in \Cref{fig:array_mediator} assumes the underlying causal mechanisms driving student behavior are symmetric with respect to permutations of the rows and of the columns, i.e. neither the ordering of the students nor the ordering of the topics matters.
The group of two separate permutations is denoted $\Sb^2$.

We will use these three examples as illustrations in our discussion of GCMs, but they are just examples. 
We are interested in GCMs for any structured variables, any causal graph, and any symmetry.

\paragraph{Formalizing geometric causal models.} To formalize GCMs, 
we relax the standard assumptions of structural causal models, so they can generate non-i.i.d. data. 
We rely on the tools of group theory, and assume the set of transformations $\Gb$ forms a symmetry group (\Cref{sec:background}). The GCM recipe, in brief, is:
\begin{enumerate}
\item We write a structural causal model as usual, but we assume the model enjoys a symmetry. The symmetry is built into the causal mechanisms, which take structured variables as arguments. The latent noise is i.i.d. at each $\omega$. 
	\item Marginalizing out the noise, we obtain a causal graphical model whose joint distribution is non-i.i.d. across $\omega$, and which respects the symmetry.
\item We define causal estimands as usual, and apply do-calculus to identify them as usual.
	\item We estimate the causal effect by incorporating the symmetry into a probabilistic model, and fitting it to the data. 
\end{enumerate}

The paper is organized as follows.
We first define causal mechanisms that are equivariant to a given symmetry group (\Cref{sec:background}).
Next, we compile individual mechanisms into geometric causal models (\Cref{sec:gcm}).
We study identification and estimation (\Cref{sec:id_estimate_ex}).
We find, for example, that for each of the GCMs in \Cref{fig:examples} we can identify the effect of the treatment $\ab$ on the outcome $\yb$. 
Next, we provide a detailed theoretical justification for our identification and estimation procedures (\Cref{sec:theory}).
Finally, we illustrate an application of GCMs to genomics, using the symmetries of DNA and geometric deep learning methods (\Cref{sec:application}).

Code for all empirical results is available at \url{https://github.com/EWeinstein/GCM}.

\section{Background: Symmetries and Equivariance} \label{sec:background}

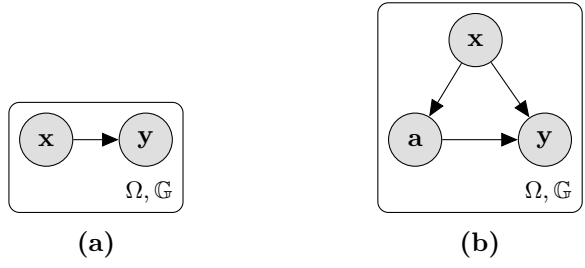
\begin{figure}[t]
\centering
\begin{subfigure}[t]{0.3\textwidth}
\centering
\begin{tikzpicture}

\node[obs]                               (y) {$\yb$};
  \node[obs, left=.6cm of y] (x) {$\xb$};

\edge {x} {y} ;

\plate{in} {(y)(x)} {$\Omega, \Gb$};

\end{tikzpicture}
\caption{} \label{fig:single_arrow}
\end{subfigure}
\begin{subfigure}[t]{0.3\textwidth}
\centering
\begin{tikzpicture}

\node[obs]                               (y) {$\yb$};
  \node[obs, left=1cm of y] (z) {$\ab$};
  \node[obs, above=.6cm of z, xshift=.8cm] (x) {$\xb$};

\edge {x,z} {y} ;
  \edge {x}{z};

\plate{in} {(y)(z)(x)} {$\Omega, \Gb$};

\end{tikzpicture}

\caption{} \label{fig:basic_example}
\end{subfigure}
\caption{\textbf{Basic GCMs.} (a) A treatment $\xb$ affecting an outcome $\yb$. (b) A confounder $\xb$ affecting a treatment $\ab$ and an outcome $\yb$.}
\end{figure}

We are interested in modeling causal relationships among structured variables. 
As an example, consider tree cover $\xb = (x_\omega)_{\omega \in \Omega}$ and temperature $\yb = (y_\omega)_{\omega \in \Omega}$ across a city, $\Omega = \Rb^2$.
To understand the impact that tree cover has on temperature, we could model the relationship as,
\begin{equation}
	\yb = \f(\xb)
\end{equation}
where $\f:\Xc^\Omega \to \Yc^\Omega$ is a function mapping between structured variables (\Cref{fig:single_arrow}). 

We aim to learn about $\f$ using data from a single city, meaning just one $\xb$ and one $\yb$. 
With no constraints on $\f$, the problem is hopeless. 
But treating each location in the city as independent, so that $y_\omega$ just depends on $x_\omega$, is unrealistic. 
Temperature at one location, $y_\omega$, could be impacted by a park nearby, $x_{\omega'}$.

Instead, we postulate that while there might be dependence across different locations in the two variables, there is no dependence on the \textit{absolute} location, only on the relative location. So if we shift the entire input $\xb$ by $\tau$, the output $\yb$ also shifts by $\tau$. Mathematically, for any $\tau \in \Rb^2$,
\begin{equation} \label{eqn:equivariant_example}
	(y_{\omega - \tau})_{\omega \in \Omega} = \f((x_{\omega-\tau})_{\omega \in \Omega}).
\end{equation}
Intuitively, this implies there is nothing special about a particular point $\omega$ in the plane $\Omega$. All that matters for $\yb$ is the surrounding context of $\xb$.

\Cref{eqn:equivariant_example} is an instance of \textit{equivariance}, where transforming a function's input transforms the output in the same way.
In this case the transformation is a shift in location.
In models of sequences, we may have equivariance to permutations. In models of arrays, we may have equivariance to permutations of the rows or the columns.
GCMs postulate that the causal mechanisms underlying data respect some equivariance.

\paragraph{Groups and actions.} These ideas are made precise using group theory. \begin{definition}[Group] 
	A group $(\Gb, \cdot)$ with elements $\phi \in \Gb$ and binary operator $\cdot$, (a) is closed under composition: $\phi \cdot \phi' \in \Gb$ for all $\phi, \phi' \in \Gb$, (b) is associative: $\phi \cdot (\phi' \cdot \phi'') = (\phi \cdot \phi') \cdot \phi''$ for all $\phi,\phi',\phi'' \in \Gb$ (c) has an identity $\ID \in \Gb$ such that $\phi \cdot \ID = \ID \cdot \phi = \phi$ for all $\phi \in \Gb$ and (d) has a unique inverse $\phi^{-1} \in \Gb$ for every element $\phi \in \Gb$, satisfying $\phi^{-1}\cdot \phi = \phi \cdot \phi^{-1} = \ID$.
\end{definition} 
\noindent We study groups of transformations. That is, the elements of the group can transform elements of another set, such as indices $\omega \in \Omega$ or variables $\xb \in \Xc^\Omega$.
The \textit{action} of an element $\phi$ specifies this transformation.
For example, actions can change the location of points in space or of entries in a table.
\begin{definition}[Group action] \label{def:transform_group}
	The elements $\phi \in \Gb$ of a group $(\Gb, \cdot)$ act on a set $S$, i.e. $\phi: S \to S$, in a manner that follows the operations of the group: $\phi'(\phi(s)) = (\phi' \cdot \phi)(s)$ for all $\phi,\phi'\in \Gb$, $s \in S$, and $\ID (s) = s$ for $s \in S$.
\end{definition}
\noindent An example is the group of shifts, $\phi_\tau(\omega) = \omega + \tau$, where $\omega,\tau \in \Rb^d$. Shifts compose: $(\phi_\tau \cdot \phi_{\tau'})(\omega) = \phi_{\tau + \tau'}(\omega) = \omega + \tau + \tau'$, associate: $(\phi_{\tau} \cdot (\phi_{\tau'} \cdot \phi_{\tau''}))(\omega) = ((\phi_{\tau} \cdot \phi_{\tau'}) \cdot \phi_{\tau''})(\omega) = \omega + \tau + \tau' + \tau''$, include an identity: $\ID(\omega) = \phi_{0}(\omega) = \omega + 0$, and have an inverse: $(\phi_\tau \cdot \phi_{-\tau})(\omega) = \phi_{0}(\omega) = \omega$. We notate this \textit{translation group} as $\Tb^d = \{\phi_\tau : \tau  \in \Rb^d\}$.

Some other groups of transformations include,
\begin{itemize}
	\item \textit{Sequence permutations.} For sequence data, consider the group of permutations $\Sb$.
Each element $\phi_\pi \in \Sb$ acts on $\omega \in \Nb$ as $\phi_\pi(\omega) = \pi_\omega$, for a permutation $\pi$.
\item \textit{Array permutations.} For array data, consider the group of transformations that separately permute the rows and columns, $\Sb^2$.
Each element $\phi_{\pi,\pi'} \in \Sb^2$ acts on $\omega \in \Nb^2$ as $\phi_{\pi,\pi'}(\omega) = (\pi_{\omega_1}, \pi'_{\omega_2})$, for permutations $\pi$ and $\pi'$.
\end{itemize}

These examples illustrate how groups can act on indices $\omega \in \Omega$, but we can apply the same groups to transform structured variables. 
\begin{definition}[Index transform] \label{def:index_transform}
	We call the action of $\phi \in \Gb$ on $\xb \in \Xc^\Omega$ an ``index transform'' if it can be written,
	\begin{equation} \label{eqn:index_transform}
	\phi(\xb) = (x_{\phi^{-1}(\omega)})_{\omega \in \Omega}.
\end{equation}
\end{definition}
\noindent For example, consider what would happen if we shifted a spatial variable $\xb \in \Xc^{\Rb^2}$ by a distance $\tau$, i.e. $\phi_\tau(\xb)$. The value of the shifted variable at position $\omega$ should equal the value of the unshifted variable at the position $\omega - \tau$. Meaning, $(\phi_\tau(\xb))_\omega = x_{\omega - \tau} = x_{\phi_{-\tau}(\omega)}=x_{\phi^{-1}_\tau(\omega)}$. So, the group of shifts is an index transform of spatial variables.

Not all actions are index transformations. Some groups flip, spin, or squash $\xb$. However, index transforms play a central role in our causal theory. Our exposition will focus on index transforms, and we will assume throughout \Cref{sec:gcm,sec:id_estimate_ex} that our group just consists of index transforms.
\begin{assumption}[Group of index transforms]\label{asm:index_transform}
	Every element $\phi \in \Gb$ is an index transform of $\Xc^\Omega$. \end{assumption}
\noindent In the theory section (\Cref{sec:theory}) we relax this assumption to consider more general actions on $\Xc^\Omega$, such as rotations and reflections.

\paragraph{Equivariance.} GCMs postulate that the relationship between cause and effect is \textit{equivariant} with respect to a group of transformations.
\begin{definition}[Equivariant functions]
A function $\f: \Xc^\Omega \to \Yc^\Omega$ is equivariant with respect to a group of transformations $\Gb$ if for all $\phi \in \Gb$,
	\begin{equation}
	\phi(\f(\xb)) = \f(\phi(\xb)).
\end{equation}
\end{definition}

For example,
\begin{itemize}
	\item 
\textit{Sequence permutations.} A function $\f(\xb) = (\g(x_\omega))_{\omega \in \Nb}$ that applies the same transformation to each element of $\xb$ is equivariant to permutations,
$$f(\phi_\pi(\xb)) = (g(x_{\pi^{-1}(\omega)}))_{\omega \in \Omega} =  (f(\xb)_{\pi^{-1}(\omega)})_{\omega \in \Omega} = \phi_\pi(f(\xb)).$$

\item \textit{Spatial shifts.} A function $\f(\xb)= (\g(x_{\ball}))_{\omega \in \Rb^d}$ that applies the same transformation to each neighborhood $\ball = \{\|\omega - \omega'\| \le \delta\}$ in $\xb$ is equivariant to shifts
$$\f(\phi_\tau(\xb)) = (\g(x_{\Nc(\omega-\tau)}))_{\omega \in \Rb} = \phi_\tau(\f(\xb)).$$
When $\g$ is linear, this equivariant function is a \textit{convolution}.
\end{itemize}
Building and learning equivariant functions is a central topic of geometric deep learning, where the notion of convolution has been generalized to arbitrary groups of transformations, extending convolutional neural networks \citep{Cohen2016-oy,Kondor2018-pz}. 
We employ these ideas to define and learn causal models.

\section{Geometric Causal Models} \label{sec:gcm}

We construct geometric causal models (GCMs).  \Cref{sec:gscm} defines GCMs at the level of structural equations, with deterministic mechanisms and stochastic noise. 
\Cref{sec:gcgm} derives causal graphical models that use stochastic mechanisms.
 
\subsection{Geometric Structural Causal Models} \label{sec:gscm}

GCMs are causal models of structured variables, such as sequences, arrays, networks or spatial processes. GCMs allow for complex dependencies within and between causal variables, giving rise to non-i.i.d. data. But they postulate an underlying regularity: the causal mechanisms are equivariant with respect to a group of transformations $\Gb$.

For example, consider a GCM with the causal graph in \Cref{fig:basic_example}. The \textit{endogenous} variables $\xb$, $\ab$, $\yb$ that appear in the graph are each structured variables. The graph shows that the treatment $\ab$ is affected by the confounder $\xb$, while the outcome variable $\yb$ is causally affected by both the treatment $\ab$ and the confounder $\xb$.
The underlying structural equations of the GCM are
\begin{align}
	\epsilon^\xb_\omega \overset{iid}\sim \pr(\epsilon^\xb) \quad \quad \xb &= \f^\xb(\epsb^\xb)\\
	\epsilon^\ab_\omega \overset{iid}\sim \pr(\epsilon^\ab) \quad \quad \ab &= \f^\ab(\xb, \epsb^\ab) \label{eqn:GSCM_a}\\ 
	\epsilon^\yb_\omega \overset{iid}\sim \pr(\epsilon^\yb) \quad \quad \yb &= \f^\yb(\ab, \xb, \epsb^\yb).
\end{align}
Each endogenous variable is generated by a deterministic causal \textit{mechanism}, $\f$, which is equivariant: $\f^
\yb(\phi(\ab), \phi(\xb), \phi(\epsb^\yb)) = \phi(\f^\yb(\ab, \xb, \epsb^\yb))$.
Each $\f$ takes as input parent variables in the graph and exogenous noise $\epsb$.
The value of the noise at each $\omega \in \Omega$ is drawn i.i.d., but equivariant mechanisms can mix this information to create dependence across positions in  $\xb$, $\ab$ and $\yb$.

Standard structural causal models are a special case, where the variables are sequences, $\Omega = \Nb$, and the mechanisms are equivariant to permutations, $\Gb = \Sb$:
\begin{align}
	\epsilon^x_\omega \sim \pr(\epsilon^x) \quad &\quad x_\omega = \g^x(\epsilon^x_\omega) \label{eqn:gscm-seq-x}\\
	\epsilon^a_\omega \sim \pr(\epsilon^a) \quad &\quad a_\omega = \g^a(x_\omega, \epsilon^a_\omega) \label{eqn:gscm-seq-a}\\
	\epsilon^y_\omega \sim \pr(\epsilon^y) \quad &\quad y_\omega = \g^y(a_\omega, x_\omega, \epsilon^y_\omega) \label{eqn:gscm-seq-y}
\end{align}
for $\omega \in \Nb$. In this special case, the causal mechanisms act separately on each unit $\omega$, so each datapoint $(x_\omega, a_\omega, y_\omega)$ is drawn i.i.d.

Stepping beyond i.i.d., consider a GCM for 1D spatial data, $\Omega = \Rb$, whose mechanisms are equivariant to shifts, $\Gb = \Tb$:
\begin{align}
	\epsilon^x_\omega \sim \pr(\epsilon^x) \quad &\quad x_\omega = \g^x(\epsilon^x_{\Nc(\omega)})\\
	\epsilon^a_\omega \sim \pr(\epsilon^a) \quad &\quad a_\omega = \g^a(x_{\Nc(\omega)}, \epsilon^a_{\Nc(\omega)})\\
	\epsilon^y_\omega \sim \pr(\epsilon^y) \quad &\quad y_\omega = \g^y(a_{\Nc(\omega)}, x_{\Nc(\omega)}, \epsilon^y_{\Nc(\omega)}),
\end{align}
for $\omega \in \Rb$, where $\Nc(\omega) = \{\omega': |\omega' - \omega| \le 1\}$ is a spatial neighborhood of $\omega$. 
This model generates correlation and interference across space. Neighbors $x_\omega$ and $x_{\omega + 1}$ are correlated, since both depend on $\epsilon^x_{\omega} \ldots \epsilon^x_{\omega+1}$. The treatment at one location, $a_\omega$, affects the outcome at another, $y_{\omega+1}$.

\paragraph{Theory} We now state a general definition of geometric structural causal models. Let $\xb^{\pa(v)}$ denote the parents of a variable $\xb^v$ in the causal graph, for all $v \in [V] \triangleq \{1, \ldots, V\}$.

\begin{definition}[Geometric structural causal model] \label{def:GSCM}
A geometric structural causal model (GSCM) $\Mc^{\scm}$ with a group of index transforms $(\Gb, \cdot)$ is defined by 
\begin{enumerate}
\item endogenous structured variables $\xb^{1}, \ldots,\xb^V$, where $\xb^v \in (\Xc^v)^\Omega$ for $v \in [V]$,
\item a directed acyclic graph $\Gc$, \item distributions $\pr(\epsilon^v)$ over exogenous noise variables $\epsb^v$ for $v \in [V]$, and 
\item causal mechanisms $\f^v$ that are equivariant to $\Gb$,
\begin{equation} \label{eqn:mechanis-equiv}
	\f^v(\phi(\xb^{\pa(v)}), \phi(\epsb^v)) = \phi(\f^v(\xb^{\pa(v)}, \epsb^v))
\end{equation} for all $\phi \in \Gb$ and $v \in [V]$.
\end{enumerate}
Each endogenous variable is generated as,
\begin{equation} \label{eqn:gcm_generate}
	\epsilon_\omega^v \overset{iid}\sim \pr(\epsilon^v) \quad \quad \xb^v = \f^v(\xb^{\pa(v)}, \epsb^v),
\end{equation}
for $\omega \in \Omega$ and $v \in [V]$.
\end{definition}
\noindent The key condition is equivariance (\Cref{eqn:mechanis-equiv}). In a standard causal model, the causal mechanisms must act separately on each unit $\omega$, so we could simplify \Cref{eqn:gcm_generate} to $x^v_\omega = \f^v(x^{\pa(v)}_\omega, \epsilon^v_\omega)$.
In a GSCM, this simplification does not hold. We assume a deeper symmetry.
Equivariance permits the model to describe interference and correlation across units. 

In all other respects, GSCMs work just like conventional SCMs. 
They can describe observations, interventions and counterfactuals \citep{Bareinboim2021-ro}. In this article, we focus mainly on using observations to learn about hard interventions.
\begin{definition}[GSCM after a hard intervention] \label{def:intervene}
	Under a hard intervention $\rmdo(\xb^v=\xb^v_\star)$, the causal mechanism $\xb^v = \f^v(\xb^{\pa(v)}, \epsb^v)$ in $\Mc^{\scm}$ is replaced by $\xb^v = \xb^v_\star$.
\end{definition}
\noindent The intervention may give every unit the same treatment, $\xb^v_{\star,\omega}=x_\star^v$ for $\omega \in \Omega$, or it may deliver treatments that vary across space, across time or across nodes in a network.

Which variables should be included in a GSCM? Variables with one or no children can be marginalized out of a GSCM, by absorbing their mechanism and noise into their child (\Cref{apx:marginalize})~\citep[][Def. 5,6]{Spirtes2010-cu,Richardson2002-sp,Janzing2022-rn}.
This is because the composition of equivariant functions is equivariant.
Variables with more than one child, \textit{confounders}, cannot be marginalized out.
These rules are the same as for conventional SCMs.

\subsection{Geometric Causal Graphical Models} \label{sec:gcgm}

Structural causal models describe causal processes using deterministic mechanisms and stochastic noise. \textit{Causal graphical models}
(CGMs), also known as \textit{causal Bayesian networks} or \textit{agnostic causal DAGs}, use stochastic mechanisms~\citep{shalizi2013advanced,Lauritzen2001-vi,Bareinboim2021-ro,Richardson2013-yy}. An SCM can be rewritten as a CGM by marginalizing out the exogenous noise.
We apply the same logic to GCMs, and integrate out the noise of GSCMs to obtain geometric causal graphical models (GCGMs). 

When we integrate out the noise in the mechanism generating $\xb$ in \Cref{fig:basic_example}, we obtain the distribution,
\begin{equation}
	\xb \sim \Prob(\xb \in \cdot) \triangleq \int \Ib(\f^\xb(\epsb^\xb) \in \cdot) \pr(\epsb^\xb)\di \epsb^\xb.
\end{equation}
Because $\f^\xb$ is equivariant and $\pr(\epsb^\ab)$ is i.i.d., this distribution is \textit{invariant} to transformations. \begin{definition}[Invariant distributions] \label{def:invariant_dist}
	A distribution $\pr(\xb)$ is invariant with respect to a group of transformations $\Gb$ when $\phi(\xb)$ has the same distribution as $\xb$, i.e.
	\begin{equation}
	\pr(\phi(\xb)) = \pr(\xb)
\end{equation}
for all $\xb \in \Xc^\Omega$ and $\phi \in \Gb$.
\end{definition} 
\noindent To be clear, just because the distribution is invariant, does not mean the variable itself is invariant, i.e. $\phi(\xb)$ need not equal $\xb$, it just has equal probability.

Next consider the mechanism generating $\ab$ in \Cref{fig:basic_example}. We obtain the conditional distribution,
\begin{equation}
	\ab \sim \Prob(\ab \in \cdot \mid \xb) \triangleq \int \Ib(\f^\ab(\xb, \epsb^\ab) \in \cdot) \pr(\epsb^\ab)\di \epsb^\ab.
\end{equation}
Since $\f^\ab$ is equivariant and $\pr(\epsb^\ab)$ is i.i.d., this distribution is \textit{conditionally equivariant}. \begin{definition}[Conditionally equivariant distributions]
	A distribution $\pr(\ab \mid \xb)$ is conditionally equivariant with respect to a group of transformations $\Gb$, if
	\begin{equation}
	\pr(\ab \mid \xb) = \pr(\phi(\ab) \mid \phi(\xb))
\end{equation}
for all $\ab \in \Ac^\Omega$, $\xb \in \Xc^\Omega$ and $\phi \in \Gb$.
\end{definition} 
\noindent In words, if we transform the input $\xb$ by $\phi$, the conditional distribution of the output $\ab$ transforms in the same way. Invariance is a special case of conditional equivariance, conditioning on nothing. 

Now, integrating out the noise in each mechanism of \Cref{fig:basic_example}, we obtain a GCGM,
\begin{align}
	\xb &\sim \pr(\xb) \quad \quad \quad \quad \quad \, \, \text{ where } \pr(\xb) = \pr(\phi(\xb))\\
	\ab &\sim \pr(\ab \mid \xb) \quad \quad \quad \quad \text{ where }\pr(\ab \mid \xb) = \pr(\phi(\ab) \mid \phi(\xb))\\ 
	\yb &\sim \pr(\yb \mid \ab, \xb) \quad \quad \quad \text{ where } \pr(\yb \mid \ab, \xb) = \pr(\phi(\yb) \mid \phi(\ab), \phi(\xb)).
\end{align}
The joint distribution over the endogenous variables is thus invariant: $\pr(\yb,\ab,\xb) = \pr(\phi(\yb),\phi(\ab),\phi(\xb))$. 

To illustrate, we return to the setting of 1D spatial data under the group of shifts, $\Omega = \Rb$,  $\Gb = \Tb$. 
Consider a GSCM with additive Gaussian noise,
\begin{align}
\epsilon^x_\omega \overset{iid}\sim \mathrm{Normal}(0, 1) \quad \quad x_\omega &= \f^x(\epsb^x)_\omega = \int A^x(\omega - \omega')\epsilon^x_{\omega'} \di\omega'\\
\epsilon^a_\omega \overset{iid}\sim \mathrm{Normal}(0, 1) \quad \quad a_\omega &= \f^a(\xb, \epsb^a)_\omega = \g^a(x_{\Nc(\omega)}) + \int A^a(\omega - \omega')\epsilon^a_{\omega'} \di\omega'\\
\epsilon^y_\omega \overset{iid}\sim \mathrm{Normal}(0, 1) \quad \quad y_\omega &= \f^y(\ab, \xb, \epsb^y)_\omega = \g^y(a_{\Nc(\omega)}, x_{\Nc(\omega)}) + \int A^y(\omega - \omega')\epsilon^y_{\omega'} \di\omega'
\end{align}
for $\omega \in \Rb$ and functions $A^x, A^a, A^y: \Rb \to \Rb$. The functions $A^v$ linearly weight the contribution of noise at different locations depending on their distance from $\omega$.
Integration gives the GCGM,
\begin{align} 
	\xb &\sim \GPc(0, k^x)\\
	\ab &\sim \GPc((\g^a(x_{\Nc(\omega)}))_{\omega \in \Rb}, k^a)\\
	\yb &\sim \GPc((\g^y(a_{\Nc(\omega)}, x_{\Nc(\omega)}))_{\omega \in \Rb}, k^y) \label{eqn:gcgm_ex_spatial}
\end{align}
where $\GPc(\mu, k^v)$ denotes a Gaussian process with mean function $\mu$ and kernel $k^v(\omega, \omega') \triangleq \int A^v(\omega - \tilde \omega) A^v (\omega' - \tilde \omega) \di \tilde \omega $. 
Here, the mean of each GP depends on a sliding window of the variable's parents, and so is equivariant to shifts of the parents. The GP adds noise around the mean, with spatial autocorrelation. 
The GP's kernel is \textit{stationary}, i.e. it just depends on the distance between points.
So, the causal mechanisms are conditionally equivariant.

\paragraph{Theory}
To summarize, we define GCGMs for arbitrary groups and graphs.
\begin{definition}[Geometric causal graphical model] \label{def:GCGM}
	Consider a geometric structural causal model $\Mc^\scm$ (\Cref{def:GSCM}). The corresponding causal graphical model has stochastic mechanisms
	\begin{equation}  \label{eqn:gcgm_mechanism}
	\xb^v \sim \Prob(\xb^v \in \Xi \mid \xb^{\pa(v)}) \triangleq \int \Ib(\f^v(\xb^{\pa(v)}, \epsb^v) \in \Xi) \pr(\epsb^v)\di \epsb^v
\end{equation}
for $v \in [V]$. 
\end{definition}
\begin{proposition} \label{prop:stochastic_mech_equiv}
The stochastic mechanisms (\Cref{eqn:gcgm_mechanism}) are conditionally equivariant.
\end{proposition}
\noindent Proof in \Cref{apx:conditional-equiv}.
Interventions in GCGMs work just as in GSCMs, with the mechanism for the intervened variable replaced by its intervened value (\Cref{def:intervene}).

\section{Identification and Estimation} \label{sec:id_estimate_ex}

The goal of causal modeling is to understand the effects of interventions. 
We observe the world as it is, but want to predict the result of a change.
Causal inference methods reason from observation to intervention by first asking what we could learn from infinite data, and then approximating it from the finite data we have at hand. The usual logic is:
\begin{enumerate}
	\item \textit{Observation.} With unlimited observational data, $x^{\Vc_\obs}_{1:\infty} \overset{iid}{\sim}\pr(x^{\Vc_\obs})$, we could learn the distribution underlying the observed causal variables, $\pr(x^{\Vc_\obs})$. Some variables may be hidden, however, so we do not necessarily know the joint over all variables $\Vc$.
	\item \textit{Identification.} Using the observational distribution and our model, we compute what would happen if we intervened on the system. We write the effect of the intervention as a function of the observational distribution, $\psi(\pr(x^{\Vc_\obs}))$. \item \textit{Estimation.} Based on the finite observational data we have at hand, we estimate the effect $\psi(\pr(x^{\Vc_\obs}))$ and quantify our remaining uncertainty.
\end{enumerate}

We extend this logic to GCMs.
\begin{enumerate}
	\item \textit{Observation.} With a complete sample from the observational distribution, $\xb^{\Vc_\obs}\sim \pr(\xb^{\Vc_\obs})$, we could learn the underlying distribution.
	\item \textit{Identification.} From the observational distribution and the model, we compute the effect of an intervention, $\psi(\pr(\xb^{\Vc_\obs}))$.
	\item \textit{Estimation.} In practice we have finite observational data, meaning we observe only part of $\xb^{\Vc_\obs}$. Based on this data, we estimate $\psi(\pr(\xb^{\Vc_\obs}))$ and quantify uncertainty.
\end{enumerate}

What justifies this reasoning, and when can it be applied? Here, we summarize the theoretical results from \Cref{sec:theory}. Then we illustrate how to apply the machinery in practice.

\subsection{Theory Overview}
\paragraph{Observation.} To learn a distribution from a single sample, $\xb^{\Vc_\obs}\sim \pr(\xb^{\Vc_\obs})$, it must have some underlying regularity.
This is captured mathematically by the notion of \textit{ergodicity}, meaning intuitively that a single realization of the system reveals its underlying laws (\Cref{def:ergodic}) \citep{Sarig2023-bz}.
We show that GCMs are ergodic under general conditions, namely when the domain $\Omega$ and group $\Gb$ are infinite, and the group's action mixes elements of the domain (\Cref{asm:mix}). 
Thanks to GCMs' ergodicity, we can infer $\pr(\xb^{\Vc_\obs})$ from a single infinite sample, $\xb^{\Vc_\obs} \sim \pr(\xb^{\Vc_\obs})$ (\Cref{thm:gcm_ergodic}).

\paragraph{Identification.} We can use \textit{do-calculus} for identification in GCMs \citep{Pearl2009-fh}.
Do-calculus computes causal effects without making assumptions about the parametric form of the model.
Though built for the usual i.i.d. settings, it ports to GCMs unchanged (\Cref{thm:docalculus}). 
We calculate $\psi(\cdot)$ from the causal graph using standard do-calculus, then plug in the structured data distribution $\pr(\xb^{\Vc_\obs})$ to obtain the GCM's effect.
This works because GCMs follow a Markov decomposition, just like usual i.i.d. models. Note this particular identification strategy assumes the causal graph and the symmetry group are both known. 

\paragraph{Estimation.} Finally we estimate the effect $\psi(\pr(\xb^{\Vc_\obs}))$ from data. 
In practice we cannot observe all of $\xb^{\Vc_\obs}$. 
Instead, we have a finite dataset, obtained by measuring a subset of the domain $A \subseteq \Omega$.

For estimation, the model must satisfy positivity, i.e. the intervention must have non-zero probability. To enable positivity we assume finite interference: the value of $x_\omega^v$ can only depend on a finite number of other indices $\omega'$ in the causal parents of $v$ (\Cref{asm:finite_depend}).
This localized dependence ensures the intervention can occur repeatedly without parametric restrictions on the model.

To guarantee consistent estimates, we employ \textit{Lindenstrauss's ergodic theorem}, an extension of the law of large numbers to ergodic distributions with general group symmetries \citep{Lindenstrauss2001-sp,Austern2018-dy}.
It requires the group is \textit{amenable}, meaning intuitively it contains a stably increasing sequence of subsets (\Cref{asm:finite-tempered}). 
It implies estimates converge to the true $\psi(\pr(\xb^{\Vc_\obs}))$ as we observe more data, i.e. as $A$ grows (\Cref{thm:estimate_id}).

For uncertainty quantification, Bayesian methods are a natural choice.
The reason is that if we assume the data distribution is invariant, this implies the existence of a prior: the ergodic decomposition theorem says the data must be generated as 
\begin{align}
	\pr &\sim \pi(\pr) \label{eqn:ergodic-prior}\\
	\xb &\sim \pr(\xb) \label{eqn:ergodic-likelihood}
\end{align}
where $\pr$ is ergodic and $\pi(\pr)$ is a prior on ergodic distributions \citep[Theorem A1.4, Farrell-Varadarajan,][]{Kallenberg2005-gv}.
When the data is a sequence and the group consists of permutations (\Cref{sec:seq-ex}), this decomposition is de Finetti's theorem, a classical motivation for Bayesian inference \citep{Orbanz2015-fm}. When the data is an array and the group consists of permutations of the rows and columns, it is the Aldous-Hoover theorem \citep{Orbanz2015-fm}. 
In general, the result says that if we assume invariance, there must be a prior. Motivated by this, we use Bayesian models that follow the ergodic decomposition (\Cref{eqn:ergodic-prior,eqn:ergodic-likelihood}) to infer the causal effect $\psi(\pr(\xb^{\Vc_\obs}))$, and quantify our uncertainty with the posterior.

\subsection{Example: Sequence Data} \label{sec:seq-ex}

We first illustrate identification and estimation in GCMs in the familiar setting of sequence data and permutation invariance. Here GCMs reduce to conventional causal models. 

We consider the GCM in \Cref{fig:examples}b. We are interested in the effect that an intervention on $\ab$ has on $\yb$. In particular, we will investigate the distribution of $\yb$ after we set $\ab = \ab_\star$. 

\paragraph{Observation.} We consider what we can learn from a single sample from the observational distribution,
\begin{equation}
	\xb, \ab, \yb \sim \pr(\xb, \ab, \yb) = \int \pr(\ub, \xb, \ab, \yb) \di \ub.
\end{equation}
We assume the mechanisms are permutation equivariant (\Cref{eqn:gscm-seq-x,eqn:gscm-seq-a,eqn:gscm-seq-y}), so the joint distribution decomposes as,
\begin{equation}
	x_\omega, a_\omega, y_\omega \sim \pr(x, a, y) \quad \quad \text{ for } \omega \in \Nb. 
\end{equation}
That is, each element of the infinite sequence is drawn i.i.d. from a common underlying distribution $\pr(x, a, y)$. So from the entire structured variables $\xb, \ab, \yb$, we can learn $\pr(x, a, y)$ and hence $\pr(\xb, \ab, \yb)$.
This is an example of ergodicity: the i.i.d. distribution $\prod_{\omega\in \Nb} \pr(x_\omega, a_\omega, y_\omega)$ is ergodic under the group of permutations, so it is learnable from a single sample.

\paragraph{Identification.} With the observational distribution $\pr(\xb, \ab, \yb)$ in hand, we aim to compute the effect of an intervention, $\pr(\yb \s \rmdo(\ab = \ab_\star))$. 
We apply do-calculus on the causal graph (\Cref{fig:seq_confounder}), which gives the back-door adjustment formula,
\begin{equation} \label{eqn:seq_id}
\pr(\yb \s \rmdo(\ab = \ab_\star)) = \psi(\pr(\xb, \ab, \yb)) = \int \pr(\yb \mid \ab_\star, \xb) \pr(\xb) \di \xb
\end{equation}
Since the distribution is i.i.d., this simplifies to,
\begin{equation}
\pr(y_\omega \s \rmdo(\ab = \ab_\star)) = \int \pr(y_\omega \mid a_{\star\omega}, x) \pr(x) \di x
\end{equation}
for $\omega \in \Nb$. This formula describes the post-intervention distribution over each element of $\yb$.
In the special case where we give every unit $\omega$ the same treatment, $a_{\star1} = a_{\star2} = \ldots = a_\star$, we recover the standard back-door formula $\pr(y \s \rmdo(a = a_\star)) = \int \pr(y \mid a_{\star}, x) \pr(x) \di x$.

\paragraph{Estimation.} Using the identification formula, we can estimate the causal effect from data. 
We observe the variables $\xb,\ab,\yb$ at a finite subset of the domain, $A \subset \Omega = \Nb$.
Thanks to permutation invariance, we assume without loss of generality that $A=\{1, \ldots, n\}$.
Our dataset is then $\Dc = (x_{1:n}, a_{1:n}, y_{1:n})$.

We use the data to estimate the causal effect. One approach is to estimate each term of the identification formula separately.
The term $\pr(\yb \mid \ab, \xb)$ breaks down into $\pr(\yb \mid \ab, \xb)= \prod_{\omega=1}^\infty \pr(y_\omega \mid a_\omega, x_\omega)$, so we can learn it by regressing $y_\omega$ on $a_\omega$ and $x_\omega$.
For the term $\pr(\xb) = \prod_{\omega=1}^\infty \pr(x_\omega)$, we can estimate the distribution of the $x_\omega$. 
Then we plug each estimate into \Cref{eqn:seq_id}.
Many other estimation approaches have also been developed for this i.i.d. setting, for example by reparameterizing the model in terms of a propensity score \citep{Kennedy2024-ac}.

\subsection{Example: Spatial Data}

We now consider causal inference from spatial data (\Cref{fig:examples}d). The domain is the plane, $\Omega = \Rb^2$, and the symmetry is the group of shifts $\Gb = \Tb^2$.
The causal graph is the same as in the previous example, and we investigate the same effect: the distribution of $\yb$ after an intervention sets $\ab = \ab_\star$.

\paragraph{Observation.} We first consider what we could learn from a single sample from the observational distribution,
\begin{equation}
	\xb, \ab, \yb \sim \pr(\xb, \ab, \yb) = \int \pr(\ub, \xb, \ab, \yb) \di \ub.
\end{equation}
Here $\xb, \ab, \yb$ are each spatial processes. 
With access to $\xb, \ab, \yb$ over the entire domain $\Omega=\Rb^2$, we can learn the  ergodic distribution $\pr(\xb, \ab, \yb)$.

\paragraph{Identification.} From the observational distribution $\pr(\xb, \ab, \yb)$ we aim to compute the effect of an intervention, $\pr(\yb \s \rmdo(\ab = \ab_\star))$. 
We apply do-calculus to \Cref{fig:examples}d, deriving the  back-door correction,
\begin{equation} \label{eqn:space_id}
\pr(\yb \s \rmdo(\ab = \ab_\star)) = \int \pr(\yb \mid \ab_\star, \xb) \pr(\xb) \di \xb.
\end{equation}

\paragraph{Estimation.}
We observe the variables $\xb, \ab, \yb$ at subset of the domain $A \subseteq \Rb^2$. 
Based on this data, we learn a probabilistic model of $\pr(\yb \mid \ab, \xb)$. 
As an example, consider a conditionally equivariant Gaussian process model (\Cref{eqn:gcgm_ex_spatial}),
\begin{equation}
	\yb \sim \GPc(\beta^a \ast \ab + \beta^x \ast \xb, k^y) \label{eqn:spatial-gp-model}
\end{equation}
where $(\beta \ast \xb)_{\omega} = \int \beta_{\omega'} x_{\omega - \omega'} \di \omega' $ denotes a linear convolution and the kernel is stationary, i.e. $k^y(\omega, \omega') = \kappa^y(\omega - \omega')$ for a function $\kappa^y$.
The convolutions generate interference across positions; the kernel produces additional noise with spatial autocorrelation.
This model is like a standard linear regression, except it is translation rather than permutation equivariant.

After fitting the model, we can calculate a causal effect. As an example, we contrast the average outcome under $\rmdo(\ab = \ab_\star)$ to $\rmdo(\ab = \mathbf{0})$, i.e. giving every unit a treatment of 0. Thanks to linearity, the contribution of $\pr(\xb)$ to the identification formula cancels, and we obtain just the convolution,
\begin{equation} \label{eqn:spatial-te}
	\Eb[\Yb \s \rmdo(\ab = \ab_\star)] - \Eb[\Yb \s \rmdo(\ab = \mathbf{0})] = \beta^a \ast \ab_\star.
\end{equation}

One way to learn the effect from data is with Bayesian inference, placing priors on $\beta^a, \beta^x$ and $k^y$.  The posterior provides both a point estimate of the treatment effect, and credible intervals.

\paragraph{Simulation study.} We illustrate with a simulation. 
We parameterize the kernel as a radial basis function (RBF), $k^v(\omega, \omega') = b^v \exp(-\eta\|\omega - \omega'\|_2^2)$, and the coefficients as Gaussian bumps, $\beta^v_\omega = c^v \exp(-\lambda \|\omega\|_2^2)$, producing autocorrelation and interference. The mechanisms generating $\ub, \xb$ and $\ab$ are assumed to have the same Gaussian process form as for $\yb$. We generate synthetic data assuming the model is well-specified.

From this simulation, we obtain a spatial dataset. We assume the variables are measured at random points in the unit square, giving $\Dc = (x_\omega, a_\omega, y_\omega)_{\omega \in A}$, where $\omega_j \overset{iid}{\sim} \mathrm{Uniform}(0, 1)$ (\Cref{fig:spatial_gcm_a,fig:spatial_gcm_x,fig:spatial_gcm_y}).
To perform inference, we place weak priors on the parameters of the coefficients and kernel. We implement the model in NumPyro, a probabilistic programming language, and perform automated approximate Bayesian inference using MCMC, the No-U Turn Sampler (NUTS) \citep{Phan2019-br,Bingham2019-aa,Hoffman2014-ic}.
We estimate the effect on the average outcome of $\ab_\star = \mathbf{1}$ versus $\ab_\star = \mathbf{0}$ by computing the posterior mean and credible intervals for $\beta^a$ (\Cref{eqn:spatial-te}).
Details are in \Cref{apx:spatial-sim}.

The GCM provides an estimate of the causal effect (\Cref{fig:spatial_gcm_inference_example}). As one comparison, we used a conventional causal model, a Bayesian linear regression, that ignores interference and autocorrelation (CM). This corresponds to setting length scales of $\eta=\lambda=0$ in $\beta^a$ and $\kappa^y$. 
As another comparison, we used a geometric model that accounts for spatial dependence but ignores confounding (GM). This corresponds to setting $\beta^x = \mathbf{0}$. 
The GCM provided a more accurate treatment effect, and a credible interval that covered the true effect (\Cref{fig:spatial_gcm_inference_example}).
This performance gain held across 20 simulation configurations, with the GCM achieving lower average error (\Cref{fig:spatial_gcm_avg_error}) and better calibrated uncertainty (\Cref{fig:spatial_gcm_coverage}). 
In summary, GCMs can produce high quality causal effect estimates by accounting for both non-i.i.d. structure and confounding in the data.

\begin{figure}[t]
	\centering
	\begin{subfigure}[t]{0.24\textwidth}
		\centering
		\includegraphics[width=\textwidth]{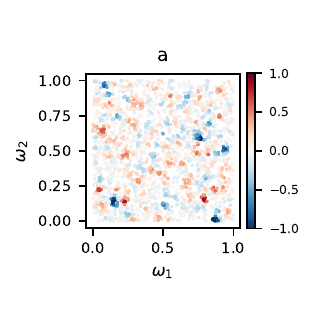}
		\caption{} \label{fig:spatial_gcm_a}
	\end{subfigure}
	\begin{subfigure}[t]{0.24\textwidth}
		\centering
		\includegraphics[width=\textwidth]{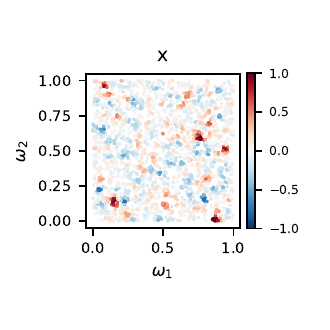}
		\caption{} \label{fig:spatial_gcm_x}
	\end{subfigure}
	\begin{subfigure}[t]{0.24\textwidth}
		\centering
		\includegraphics[width=\textwidth]{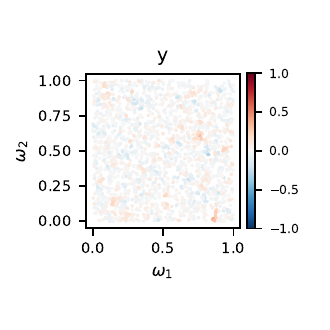}
		\caption{} \label{fig:spatial_gcm_y}
	\end{subfigure}
	\begin{subfigure}[t]{0.24\textwidth}
		\centering
		\includegraphics[width=\textwidth]{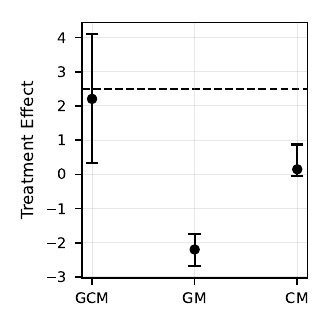}
		\caption{} \label{fig:spatial_gcm_inference_example}
	\end{subfigure}
	\caption{\textbf{Spatial GCM simulation study.} (a,b,c) Spatial data, irregularly sampled at 2000 points across the domain, consisting of the treatment $\ab$ (a), observed confounder $\xb$ (b) and outcome $\yb$ (c). (d) Estimated treatment effect, including posterior mean and credible interval (5th-95th). GCM: geometric causal model, GM: geometric model (no confounding correction), CM: conventional i.i.d. causal model.}\label{fig:spatial_gcm}
\end{figure}

\paragraph{Variations.} GCMs can also describe many variations and extensions of this spatial model. The domain could be one-dimensional, three-dimensional, or higher. The space could be discretized into a grid, $\Omega = \Zb^d$, rather than continuous. 
For time series data, we may work with the domain $\Omega = \Rb_+$ and the group of shifts $\Tb$. For spatiotemporal data, the product of space and time, $\Omega = \Rb^2 \times \Rb_+$ \citep{Christiansen2022-ip}.
The probabilistic model used for estimation could also be more flexible, e.g. using convolutions plus nonlinearities for the equivariant mechanisms $\f$ \citep{Bronstein2021-gz}.

\subsection{Example: Array Data}

We next consider causal inference from array data, such as data collected from students studying different topics, or cells expressing different genes.
This data is sometimes called a \textit{bipartite network} connecting two types of nodes, e.g. students and topics, rows and columns.
We denote the index set as $\Oarray$, so variables can describe connections, $\xb_{\omega_1 \omega_2}: \omega_1, \omega_2 \in \Nb^2$, row-specific values, $\xb_{\omega_1 \diamond}: \omega_1 \in \Nb$, column-specific values, $\xb_{\diamond \omega_2} : \omega_2 \in \Nb$, or all of the above.
We consider the group of permutations $\Sb^2$ in which the permutations act separately on the rows and columns of the array: $\phi_{\pi,\pi'}(\omega_1, \omega_2) = (\pi(\omega_1), \pi'(\omega_2))$, with the convention $\pi(\diamond) = \diamond$.

A distribution with this symmetry is called \textit{separately exchangeable}. In a GSCM with the symmetry, we can generate a single variable as,
\begin{equation}
	x_{\omega_1,\omega_2} = \g(\epsilon^x_{\omega_1 \omega_2},\epsilon^x_{\omega_1 \diamond},\epsilon^x_{ \diamond\omega_2}) \quad \quad \quad \epsilon^x_\omega \overset{iid}{\sim} \pr(\epsilon^x) \text{ for } \omega \in \Oarray.
\end{equation}
This structural equation matches the \textit{Aldous-Hoover representation} of a separately exchangeable matrix, and when $\xb$ is binary, the mechanism $\g$ is called the \textit{graphon}.
The Aldous-Hoover theorem says that all separately exchangeable ergodic distributions can be written this way \citep[Chap. 7][]{Kallenberg2005-gv}.
GCMs posit this Aldous-Hoover representation describes a causal relationship, and use it to model causal dependencies among variables.

We examine a GCM with the causal graph in \Cref{fig:examples}f.

\paragraph{Observation.} From a sample with infinite rows and columns,
\begin{equation}
	\ab, \xb, \yb \sim \pr(\ab, \xb, \yb) = \int \pr(\ub, \ab, \xb, \yb) \di \ub,
\end{equation}
we can learn the underlying ergodic distribution on arrays, $\pr(\ab, \xb, \yb)$.

\paragraph{Identification.} From the observational distribution $\pr(\xb, \ab, \yb)$ we compute the effect of an intervention, $\pr(\yb \s \rmdo(\ab = \ab_\star))$, using do-calculus,
\begin{equation}
\pr(\yb \s \rmdo(\ab = \ab_\star)) = \int \int \pr(\yb \mid \ab, \xb) \pr(\ab) \pr(\xb \mid \ab_\star) \di \xb \di \ab.
\end{equation}
This is a front-door adjustment, using the mediator $\xb$ to correct for hidden confounders $\ub$ \citep{Pearl2009-fh}.

\paragraph{Estimation.} We observe the variables $\xb, \ab, \yb$ only at a subset of the rows and columns.
Thanks to permutation invariance, we can assume without loss of generality these are the first $n$ rows and $m$ columns. 
We thus have a dataset $\Dc = (x_{1:n,1:m}, a_{1:n,1:m},y_{1:n,1:m})$.
To illustrate, we assume here $\xb$, $\ab$ and $\yb$ do not have row- or column-specific values, but the confounder $\ub$ does.

We can estimate causal effects based on this data using a probabilistic model. As an example we consider,
\begin{align}
	x_\omega &\sim \mathrm{Normal}(\gamma a_\omega + \epsilon_{\omega_1\diamond }^x + \epsilon_{\diamond\omega_2}^x, \sigma^x) \quad \quad \quad \quad \quad \quad \quad \quad \epsilon_{\omega_1\diamond }^x, \epsilon_{\diamond\omega_2}^x \sim \mathrm{Normal}(0, \tau)\\
	y_\omega &\sim \mathrm{Normal}(\beta^{a} a_\omega + \beta^x x_\omega + \epsilon_{\omega_1\diamond}^y + \epsilon_{\diamond\omega_2}^y, \sigma^y) \quad \quad \quad \quad \epsilon_{\omega_1\diamond}^y, \epsilon_{\diamond\omega_2}^y \sim \mathrm{Normal}(0, \tau)
\end{align}
for $\omega_1,\omega_2 \in \Nb^2$, where $\epsilon^x_{\omega_1\diamond}, \epsilon^x_{\diamond\omega_2}, \epsilon^y_{\omega_1\diamond}, \epsilon^y_{\diamond\omega_2}$ are per-row and per-column latent variables. Intuitively, each mechanism of this model is a linear regression with fixed effects for the row and column \citep{Wooldridge2010-dv}. The latent variables provide a low-dimensional matrix decomposition of the noise, modeling correlation across rows and columns.
The overall model is conditionally equivariant.
It estimates the treatment effect as,
\begin{equation} \label{eqn:array-te}
	\Eb[\Yb \s \rmdo(\ab = \ab_\star)] - \Eb[\Yb \s \rmdo(\ab = \mathbf{0})] = \beta^x \gamma\, \ab_\star.
\end{equation}
Here, the contribution of $\pr(\xb)$ to the identification formula cancels, and the treatment effect is the product of the effect of the treatment on the mediator, and of the mediator on the outcome.

To learn the effect from data, we place priors on $\beta^a, \beta^x, \gamma$ and the latent variables, $\epsilon^x, \epsilon^y$, and conduct Bayesian inference. We then examine the posterior over the effect. 

\paragraph{Simulation study.} We illustrate with a simulation. The mechanisms generating $\ub$, $\xb$ and $\yb$ are assumed to have the same linear plus latent form as for $\ab$. 
From this simulation, we obtain an array dataset.  
We assume the variables are measured for 200 rows (e.g. students) and 50 columns (e.g. topics), giving $\Dc = (a_{1:200,1:50}, x_{1:200,1:50}, y_{1:200,1:50})$ (\Cref{fig:array_gcm_a,fig:array_gcm_x,fig:array_gcm_y}).
To perform inference we place weak priors on the coefficients. We place priors on the latents that are wider than the data generating distribution, producing misspecification. 
We implement the model in NumPyro, and perform automated approximate Bayesian inference using stochastic variational inference, Automatic Differentiation Variational Inference (ADVI) \citep{Phan2019-br,Bingham2019-aa,Kucukelbir2017-ir}. 
This allows fast inference over hundreds of latent variables.
We estimate the effect of $\ab_\star=\mathbf{1}$ by computing the posterior mean and credible intervals of $\beta^x \gamma$ from the variational approximation. Details are in \Cref{apx:array-sim}.

The GCM provides an estimate of the causal effect (\Cref{fig:array_gcm_inference_example}).
As one comparison, we used a conventional causal model, a Bayesian linear regression, that ignores correlation across rows or columns (CM). This corresponds to setting $\epsilon^x_{\omega_1\diamond} = \epsilon^x_{\diamond\omega_2} = \epsilon^y_{\omega_1\diamond} = \epsilon^y_{\diamond\omega_2} = 0$.
As another comparison, we used a geometric causal model with the wrong causal graph, assuming $\xb$ is a confounder rather than a mediator (GM).
The GCM with the correct causal structure provides the estimate with the least error (\Cref{fig:array_gcm_inference_example}). This performance advantage held on average across 20 simulation configurations (\Cref{fig:array_gcm_avg_error}). The GCM also had better calibrated frequentist uncertainty, though calibration was poor in absolute terms, likely due to the mean-field posterior approximation (\Cref{fig:array_gcm_coverage}) \citep{Wang2019-ne}.
In sum, GCMs can produce accurate causal effect estimates by accounting for complex non-i.i.d. structure in array data.

\begin{figure}[t]
	\centering
	\begin{subfigure}[t]{0.24\textwidth}
		\centering
		\includegraphics[width=\textwidth]{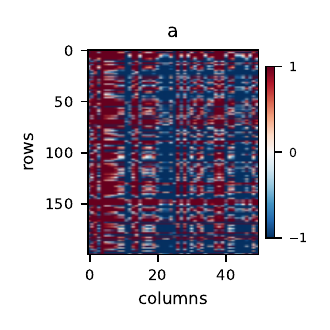}
		\caption{} \label{fig:array_gcm_a}
	\end{subfigure}
	\begin{subfigure}[t]{0.24\textwidth}
		\centering
		\includegraphics[width=\textwidth]{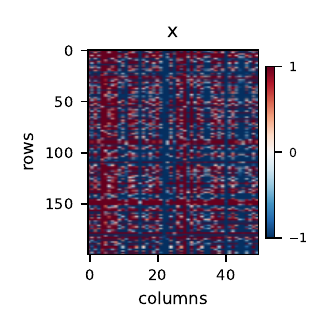}
		\caption{} \label{fig:array_gcm_x}
	\end{subfigure}
	\begin{subfigure}[t]{0.24\textwidth}
		\centering
		\includegraphics[width=\textwidth]{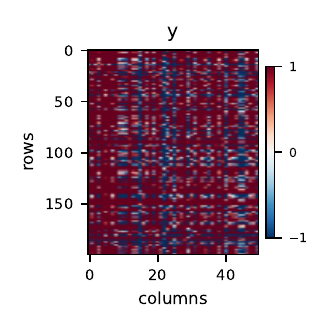}
		\caption{} \label{fig:array_gcm_y}
	\end{subfigure}
	\begin{subfigure}[t]{0.24\textwidth}
		\centering
		\includegraphics[width=\textwidth]{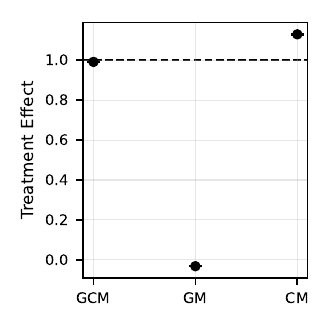}
		\caption{} \label{fig:array_gcm_inference_example}
	\end{subfigure}
	\caption{\textbf{Array GCM simulation study.} (a,b,c) Array data, with different numbers of rows and columns, consisting of the treatment $\ab$ (a), mediator $\xb$ (b) and outcome $\yb$ (c). (d) Estimated treatment effect, including posterior mean and credible interval (2.5th-97.5th). GCM: geometric causal model, GM: geometric model (incorrect causal correction), CM: conventional causal model.}\label{fig:array_gcm}
\end{figure}

\paragraph{Variations.} GCMs can also describe many variations and extensions of this array data example. 
For graph data, the row $i$ and column $j$ of an interaction matrix both correspond to the same node, so a natural symmetry is the group of joint permutations of the rows and columns, $\phi_\pi \in \Sb: \phi_\pi(\omega) = (\pi(\omega_1), \pi(\omega_2))$. Distributions invariant to this symmetry are called \textit{jointly exchangeable} \citep{Orbanz2015-fm}.
Other datasets are grouped or hierarchical, with datapoints collected from different environments \citep{Peters2016-qg,Weinstein2026-jl}.
Here, environments are exchangeable, and subunits are exchangeable within each environment, so the distribution is symmetric under the
transformation $\phi_{\pi, \pi^1 \pi^2 \ldots}(\omega) = \pi(\omega_1) \pi^{\omega_1}(\omega_2)$, for any sequence of permutations $\pi, \pi^1 \pi^2 \ldots$.
All these symmetries generalize to higher dimensions, to work with tensor data, indexing with $\Omega = (\Nb \cup \diamond)^d$ \citep{Kallenberg2005-gv}.
In this example we used linear models for estimation, but models can also be parameterized more flexibly, introducing richer matrix factorizations of the noise and incorporating nonlinearities into the causal mechanisms.

\section{Theory} \label{sec:theory}

We develop the theory of inference in geometric causal models.
First, we show that GCMs are \emph{ergodic}, so we can learn their distribution (\Cref{sec:ergodic}). 
Next we show post-intervention distributions can be identified from the observational distribution via do-calculus (\Cref{sec:do-calc}).
Finally, we show that causal effects can be estimated consistently (\Cref{sec:estimation}).

\paragraph{Generalized setting.} Our theory covers a more general setting than that in \Cref{sec:gcm}, where
the exposition focused on index transformations. GCMs extend to arbitrary group actions. This includes, for example, rotation groups, which occur widely in chemistry and molecular modeling \citep{Kohler2020-sx,Mallet2021-wt,Batatia2022-dp}.
\begin{definition}[GSCM without \Cref{asm:index_transform}]
A GSCM (\Cref{def:GSCM}) with a group $(\Gb, \cdot)$ that does not satisfy \Cref{asm:index_transform} must additionally satisfy that $\pr(\epsb^v)$ is invariant for $v \in [V]$.
\end{definition}
\noindent The added invariance condition ensures the GCM's joint distribution is invariant. 
For groups that only contain index transformations, this invariance is automatic, since any i.i.d. noise distribution $\pr(\epsb^v)$ is invariant to any index transformation.

As an example of a more general group, consider the infinite orthogonal group $\Ob$, which acts on sequences, $\Omega =\Nb$. Beyond just permutations, it contains all rotations and reflections. Formally, $\Ob= \cup_{n = 1}^\infty \Ob(n)$, where $\Ob(n)$ is the $n$-dimensional orthogonal group. Each element $\phi_{R_n} \in \Ob$ acts on $(x_1, x_2, \ldots)$ by rotating/reflecting $x_1, \ldots, x_n$ according to an orthogonal matrix $R_n$, while leaving $x_{n+1:\infty}$ unchanged: $\phi(x_{1:\infty}) = ((R_n \cdot x_{1:n}^\top)_{1:n}, x_{n+1:\infty})$.
Freedman's theorem implies that any distribution $\pr(\epsb^v)$ that is i.i.d. and invariant to this rotation group must be mean zero Gaussian, i.e. $\epsilon_\omega \sim \mathrm{Normal}(0, \sigma)$.

\subsection{Ergodicity} \label{sec:ergodic}

GCMs describe data with complex dependencies, where observations are not drawn i.i.d. 
However, we can still learn the underlying distribution because it is \textit{ergodic}. 
Here we review ergodicity under group actions; the better known concepts of ergodicity in Markov chains, or in dynamical systems, are special cases. 
Then, we show that the distribution over endogenous variables in a GCM is ergodic.

We have so far defined distributions on $\Xc^\Omega$ that are invariant under $\Gb$ (\Cref{def:invariant_dist}).
Let $\Pc_\Gb$ denote the set of all invariant distributions.
Ergodic distributions are a special type of invariant distribution that can be learned from a single sample.
Notationally, for a set $B$ and transformation $\phi \in \Gb$, define $\phi B \triangleq \{\phi(b): b \in B\}$, the transformation applied to each element.
A set $B$ is \textit{invariant} if it remains unchanged under any $\phi \in \Gb$ \citep{Sarig2023-bz}.
\begin{definition}[Ergodic distribution; App. A.1, \citep{Kallenberg2005-gv}, Chap. 1.4, \citep{Sarig2023-bz}] \label{def:ergodic}
A distribution $\pr(\xb)$ is ergodic with respect to $\Gb$ if it is invariant and $\pr(\xb \in B) \in \{0, 1\}$ for all invariant sets $B\subseteq \Xc^\Omega$.
\end{definition}
\noindent Intuitively, ergodic distributions can be learned from a single sample because they put all their mass uniformly over one invariant set. 
Once we see a sample, $\xb\sim\pr(\xb)$, we know it must be in the invariant set, $B_0 = \{\phi(\xb): \phi \in \Gb\}$, and then ergodicity implies $\pr(\xb \in B_0) = 1$.
Since ergodic distributions are invariant, $\pr(\xb)$ must be uniform over $B_0$: $\pr(\phi(\xb)) = \pr(\xb)$ for all $\xb \in B_0$ and $\phi \in \Gb$.
So from a single sample, $\xb \sim \pr(\xb)$, we can pin down $\pr(\xb)$ as the uniform distribution over all transformations of $\xb$.

We will show GCMs are ergodic.
First, we establish their noise is ergodic.
Then, we apply the fact that ergodicity is preserved under equivariant transformations \citep[Lemma A1.1]{Kallenberg2005-gv}. As a result, the causal mechanisms must generate ergodic distributions.

To establish the ergodicity of the noise, we assume 
the domain $\Omega$ is infinite and that the group's action can mix around the elements of the domain. 
Intuitively, this ensures the model generates enough diverse examples to learn the i.i.d. noise distribution nonparametrically. 
For simplicity, the condition below assumes discrete domains $\Omega$ and groups $\Gb$;
for continuous spaces $\Omega = \Rb^d$ and the group of shifts $\Gb = \Tb^d$, Maruyama's theorem guarantees ergodicity~\citep{Krishnapur2014-qs,Dym1976-ei}.
\begin{assumption}[Mixing action on $\Omega$] \label{asm:mix}
	Let $\Omega$ and $\Gb$ be countable and infinite. Assume there exists a sequence of index transforms $\phi_n \in \Gb$ such that for any finite subset of the domain, $S \subset \Omega$, $|S| < \infty$, we have $|\phi_n (S) \cap S| \to 0$ as $n \to \infty$.
\end{assumption}
\noindent Note not all elements of $\Gb$ need to be index transforms, just a subset. To illustrate the condition, we check the group of permutations, defined formally as $\Sb = \cup_{n=1}^\infty \Sb(n)$, where $\Sb(n)$ is the group of permutations that act on the first $n$ elements of the sequence.
Choose $\phi_n \in \Sb$ to be the permutation that swaps the first $n$ positions with the second $n$ positions, $\phi_n(\omega) = \omega + n \Ib(\omega \le n) - n \mathbb{I}(n < \omega \le 2n)$.
Then $|\phi_n (S) \cap S| = 0 $ for all $n \ge \max_{\omega \in S} \omega$.
Likewise for arrays, we can choose $\phi_n$ as the permutation that swaps the first $n$ rows and the first $n$ columns with the second $n$ rows and second $n$ columns.
Overall, the intuition behind \Cref{asm:mix} is that no element of $\Omega$ is privileged: we can always use transformations in the group $\Gb$ to map $\omega$ to an infinite number of other elements.

With this assumption in place we can show GCMs are ergodic.
\begin{proposition}[GCMs are ergodic] \label{thm:gcm_ergodic} Let \Cref{asm:mix} hold for the action of $\Gb$ on each $\epsb^v$ and $\xb^v$ for $v \in [V]$. Then the distribution $\pr(\xb^\Cc)$ for any $\Cc \subseteq [V]$ is ergodic.
\end{proposition} 
\noindent The proof is given in \Cref{apx:proof_ergodic}. Taking $\Cc = \Vc_\obs$, we find $\pr(\xb^{\Vc_\obs})$ is ergodic, and so we can learn the observational distribution of a GCM from a single sample.

\subsection{Identification} \label{sec:do-calc}

We can learn the observational distribution of a GCM. Now, we identify the effect of an intervention, computing it from the observational distribution and its conditionals. 

In GCMs, as with conventional causal models, we can identify with do-calculus.
\begin{proposition}[do-calculus in GCGMs] \label{thm:docalculus}
Consider the effect of an intervention $\rmdo(\ab = \ab_\star)$ on $\yb$, where $\yb$ and $\ab$ may be individual variables or, more generally, non-overlapping subsets of $\xb^\Vc$. 
If the causal effect $\pr(\yb \s \rmdo(\ab = \ab_\star))$ is identified by do-calculus on the graph $\Gc$, it can be computed as a functional of the joint distribution $\pr(\xb^{\Vc_\obs})$ and its conditionals, $\pr(\yb \s \rmdo(\ab = \ab_\star)) = \psi(\pr(\xb^{\Vc_\obs}))$, with $\psi(\cdot)$ the identification formula from do-calculus.
\end{proposition}
\noindent The proof rests on the fact that GCGMs follow a Markov decomposition according to their causal graph, just like conventional causal models (\Cref{apx:proof_do}).

\Cref{thm:docalculus} tells us that if an effect is identified by do-calculus, it is identified in a GCM. 
We now consider the reverse direction: if an effect is identified in a GCM, can we compute it by do-calculus? 
For conventional causal models, the answer is yes: do-calculus is \textit{complete}~\citep{Shpitser2006-jg}. For GCMs, the answer depends on the structure of the group action: so long as the group transforms structured variables just by index transforms, and we do not add any restrictions on the causal mechanisms, do-calculus is complete.
\begin{proposition} \label{thm:docalculus_complete}
	Assume every $\phi \in \Gb$ is an index transform (\Cref{asm:index_transform}). Then if an effect $\pr(\yb \s \rmdo(\ab = \ab_\star))$ can be computed as a functional of $\pr(\xb^{\Vc_\obs})$, it can be computed by do-calculus. 
\end{proposition}
\noindent The proof is in \Cref{apx:proof_do_complete}.

\begin{figure}[t]
\centering
\begin{tikzpicture}

\node[obs]                               (y) {$\yb$};
  \node[obs, left=.4cm of y] (a) {$\ab$};
  \node[obs, left=.4cm of a] (z) {$\zb$};
  \node[latent, above=.5cm of a, xshift=.5cm] (u) {$\ub$};

\edge {u} {y} ;
  \edge {u}{a};
  \edge{a}{y};
  \edge{z}{a};

\plate{in} {(y)(a)(z)(u)} {$\Omega=\Nb,\Gb=\Ob$};

\end{tikzpicture}
\caption{Instrumental variable (IV) GCM.} \label{fig:iv_orthogonal}
\end{figure}
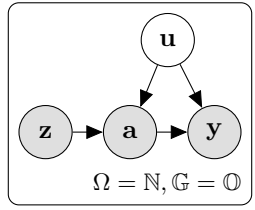

For group with actions that are not index transforms, do-calculus can be incomplete. 
For example, consider the GCM in \Cref{fig:iv_orthogonal}, which has an instrumental variable (IV) graph. The causal effect $\pr(\yb \s \rmdo(\ab_\star))$ is not identified by do-calculus \citep{Pearl2009-fh}. 
But the GCM respects the orthogonal group.
Now, the average effect becomes identified, under regularity conditions.
\begin{proposition} \label{thm:iv-id} 
	Consider the GCM in \Cref{fig:iv_orthogonal}. Assume the instrument $\zb$ is not constant, i.e. $\pr(\zb)$ is not a delta function, and the instrument is relevant, meaning it has a non-zero effect on the treatment $\ab$: $\mathbb{E}_\pr[A_\omega \s \rmdo(\zb_\star')] - \mathbb{E}_\pr[A_\omega \s \rmdo(\zb_\star)] \neq 0$ for $\zb_\star,\zb_\star' \in \Zc^\Omega, \omega \in \Omega$. Further assume interference is finite (\Cref{asm:finite_depend}, below). Then, the effect $\Eb[\Yb \s \rmdo(\ab_\star)]$ is identified.
\end{proposition} 
\noindent The proof is in \Cref{apx:proof-iv}. 
The idea is that under $\Ob$, the ergodic distributions must be i.i.d. Gaussian (Freedman's theorem), and the equivariant mechanisms must be linear \citep{Freedman1962-vh,Freedman1963-to,Orbanz2015-fm}. Causal effects are identified in linear-Gaussian IV models, when the instrument is relevant and not constant \citep{Bowden1985-ot,Pearl2009-fh}. 

This result illustrates how symmetry can give rise to causal identification.
The infinite orthogonal group encompasses the permutation group as a subgroup, $\Sb \subset \Ob$.
The causal effect is not identified under $\Sb$, but is identified under $\Ob$.
So using the same causal graph but stronger symmetries, i.e., larger groups $\Gb$, can enable causal identification.

\subsection{Estimation} \label{sec:estimation}

We next estimate causal effects from finite data.
We show that when a causal effect is identified in a GCM, we can estimate it \textit{consistently}, approximating the truth arbitrarily well as we gather more data. 
We study estimates as we observe $\xb^\obs \sim \pr(\xb^{\Vc_\obs})$ for increasing subsets of the domain, i.e., $x^\obs_{A_1}, x^\obs_{A_2}, \ldots$ for $A_n \subset \Omega$. 
For example, $A_n = \{1, \ldots, n\}$ when $\Omega = \Nb$.
To construct an estimator from this data, we average over a subset of transformations $G_n \subset \Gb$.
To simplify the presentation, we assume here $\Omega$ and $\Gb$ are countable and $\Xc$ is finite.

To enable positivity, we assume \textit{finite interference}, so variables at each index $\omega$ can only be causally impacted by a finite number of other units $\omega' \neq \omega$~\citep[Example I.1.9]{Shields1996-xl}.
\begin{assumption}[Finite interference] \label{asm:finite_depend}
	There exist subsets $D^v_\omega \subset \Omega$ with $|D^v_\omega| < \infty$ such that for all $v \in \Vc$,
	\begin{equation}
		\f^v(\xb^{\pa(v)}, \epsb^v)_\omega = \f^v(\tilde \xb^{\pa(v)}, \tilde \epsb^v)_\omega 
	\end{equation}
	when $\xb^{\pa(v)}_{D^v_\omega} = \tilde \xb^{\pa(v)}_{D^v_\omega}$ and $\epsb^v_{D^v_\omega} = \tilde \epsb^v_{D^v_\omega}$.
\end{assumption}
\noindent This assumption lets us learn conditional distribution such as $\pr(\yb \mid \ab_\star)$. 
Since $\pr(\ab)$ is ergodic, the probability of an intervention $\pr(\ab_\star)$ must be zero whenever the intervention $\ab_\star$ is not a mere transform $\phi(\ab)$ of the observation $\ab$.
By restricting interference so that $\yb$ only depends on a finite section of $\ab$, we can learn about the conditional by observing local sliding windows $D^y_\omega$. On these local windows, positivity can be satisfied, i.e. we can have $\pr(\ab_{\star D^y_\omega}) > 0$ without $\ab_\star = \phi(\ab)$.

To construct an estimator, we take an average of the data over a subset of transformations in the group $\Gb$. 
We assume that these transformations also act locally on the structured variables, and do not bring in information from infinitely far away.
To establish convergence, we also assume the group $\Gb$ contains an increasing set of transformations $G_n$ which grow smoothly with $n$ \citep{Lindenstrauss2001-sp}.
Notationally, when $\phi$ is not an index transform, we 
use $\phi^{-1}(S)$ to denote the indices that contribute to $x_S$, i.e. choose $\phi^{-1}(S)$ to be a subset of $\Omega$ such that $\phi(\xb)_S = \phi(\xb')_S$ whenever $x_{\phi^{-1}(S)} = x'_{\phi^{-1}(S)}$.
For example, if $\phi \in \Ob$ rotates $x_{1:n}$, then $\phi^{-1}([m]) = [n]$ for $m \le n$. \begin{assumption}[Amenable group with finite tempered sequence and local information] \label{asm:finite-tempered}
	Assume there exists a sequence of subsets $G_1, G_2, \ldots$, with $G_n \subset \Gb$, such that,
	\begin{enumerate}
		\item The sequence is F\o lner: for all $\phi \in \Gb$,
	\begin{equation}
		\frac{|(\phi G_n) \cap G_n|}{|G_n|} \xrightarrow{n \to \infty} 1.
	\end{equation}
	\item The sequence is tempered:
\begin{equation}
	\bigg| \underset{k < n} \bigcup G_k^{-1} G_n \bigg| \le c |G_n|
\end{equation} 
\item Transformations bring local information:  for any finite $S \subset \Omega$, we have $|\phi^{-1}(S)| < \infty$ for all $\phi \in G_n$.
	\end{enumerate}
\end{assumption}
\noindent For example, consider the group of permutations $\Sb$ and let $G_n \subset \Sb$ be the set of all permutations that modify at most the first $n$ elements. 
Now, each element $\phi \in \Sb$ is in $G_n$ for $n$ sufficiently large, so we will have $|\phi G_n \cap G_n| = 0$ eventually, and so $G_n$ is a F\o lner sequence.
Further, $G^{-1}_k G_n = G_n$ for all $k < n$, since the set of permutations of the first $k$ elements is a subset of permutations of $n$. So $G_n$ is tempered. 
Finally, every element $\phi \in \Sb$ only modifies a finite number of indices of $\xb$, so transformations in $G_n$ bring in only local information.
For arrays, we can choose $G_n$ to be the set of all permutations of the first $n$ rows and of the first $n$ columns.
For spatial data, we can choose $G_n$ to be all shifts up to a maximum distance $n$.

With these conditions in place, we show that consistent estimation is possible from finite data. The key tool is Lindenstrauss's ergodic theorem~\citep{Lindenstrauss2001-sp}, which extends the law of large numbers to ergodic distributions and group actions.
We use \Cref{asm:finite-tempered}.3 to ensure its estimator can be computed from finite data. 
\begin{proposition}[Lindenstrauss' theorem for finite data] \label{thm:lindenstrauss}
	Let \Cref{asm:finite-tempered} hold, assume $\Omega$ and $\Gb$ are countable and $\Xc$ is finite.
	Assume $\xb \sim \pr(\xb)$ for $\pr(\xb)$ ergodic, and define the estimator 
	\begin{equation} \label{eqn:lindenstrauss}
		q_{n}(x_{A_n}) \triangleq \frac{1}{|G_n|} \sum_{\phi \in G_n} \mathbb{I}[\phi(\xb)_S = \tilde x_S],
	\end{equation}
	where $A_n \triangleq \cup_{\phi \in G_n} \phi^{-1}(S)$ is guaranteed to be finite, $|A_n| < \infty$, for any finite $S \subset \Omega$.
	Then we have $q_n(x_{A_n}) \xrightarrow{n\to\infty} \pr(x_S = \tilde x_S)$ a.s..
\end{proposition}
\noindent The proof is in \Cref{apx:proof_linden}. Note that it is possible to compute the estimator $q_n(x_{A_n})$ even when we do not observe all of $\xb$, but just observe a finite dataset $x_{A_n}$. The reason is that each term $(\phi(\xb))_S$ in the sum in \Cref{eqn:lindenstrauss} depends just on the value of $\xb$ within $A_n$, a finite subset of $\Omega$.  

We can now estimate causal effects in GCMs from finite data.
Specifically, we can consistently estimate $\pr(y_S \s \rmdo(\ab_\star))$, the marginal of $\pr(\yb \s \rmdo(\ab_\star))$, for any arbitrarily large finite subset of the domain, $S \subset \Omega$. 
The final key assumption is positivity, which ensures that different treatments have a non-zero probability of occurring.
\begin{proposition} \label{thm:estimate_id}
	Let \Cref{asm:mix}, \Cref{asm:finite_depend} and \Cref{asm:finite-tempered} hold, and $\Xc^v$ is finite for $v \in \Vc$. 
	Assume the causal effect is identified by do-calculus. Assume positivity: $\pr(x^{\Vc_\obs}_I = \tilde x^{\Vc_\obs}_I) > 0$ for all $\tilde x^{\Vc_\obs}_I$ and finite $I \subset \Omega$. 
	Choose any finite $S \subset \Omega$ and let $A_n \triangleq \cup_{\phi \in G_n} \phi^{-1}(S)$. 
	There exist estimators $q_n(x^{\Vc_\obs}_{A_n} \s \rmdo(\ab_\star))$ such that $q_n(x^{\Vc_\obs}_{A_n} \s \rmdo(\ab_\star)) \xrightarrow{n\to\infty} \pr(y_S = \tilde y_S \s \rmdo(\ab_\star))$ a.s..
\end{proposition}
\noindent The proof, which constructs $q_n(x^{\Vc_\obs}_{A_n}\s \rmdo(\ab_\star))$ based on \Cref{thm:lindenstrauss}, is in \Cref{apx:proof_estimate_id}.
This result shows that causal inference in GCMs is possible. It holds for general causal graphs and general symmetries, and confirms consistent estimators can be computed from finite data.

Beyond consistency, we also can expect asymptotic normality, since the central limit theorem extends to ergodic distributions in this setting \citep{Austern2018-dy}.
This provides variance estimates and confidence intervals. 

\section{Application: Causal Inference Along the Genome} \label{sec:application}

We illustrate GCMs with an application to genomics. 
We are interested in understanding how variation in the human genome sequence impacts its biological activity.

\begin{figure}[t]
\centering
\begin{subfigure}[t]{0.4\textwidth}
		\centering
		\includegraphics[width=\textwidth]{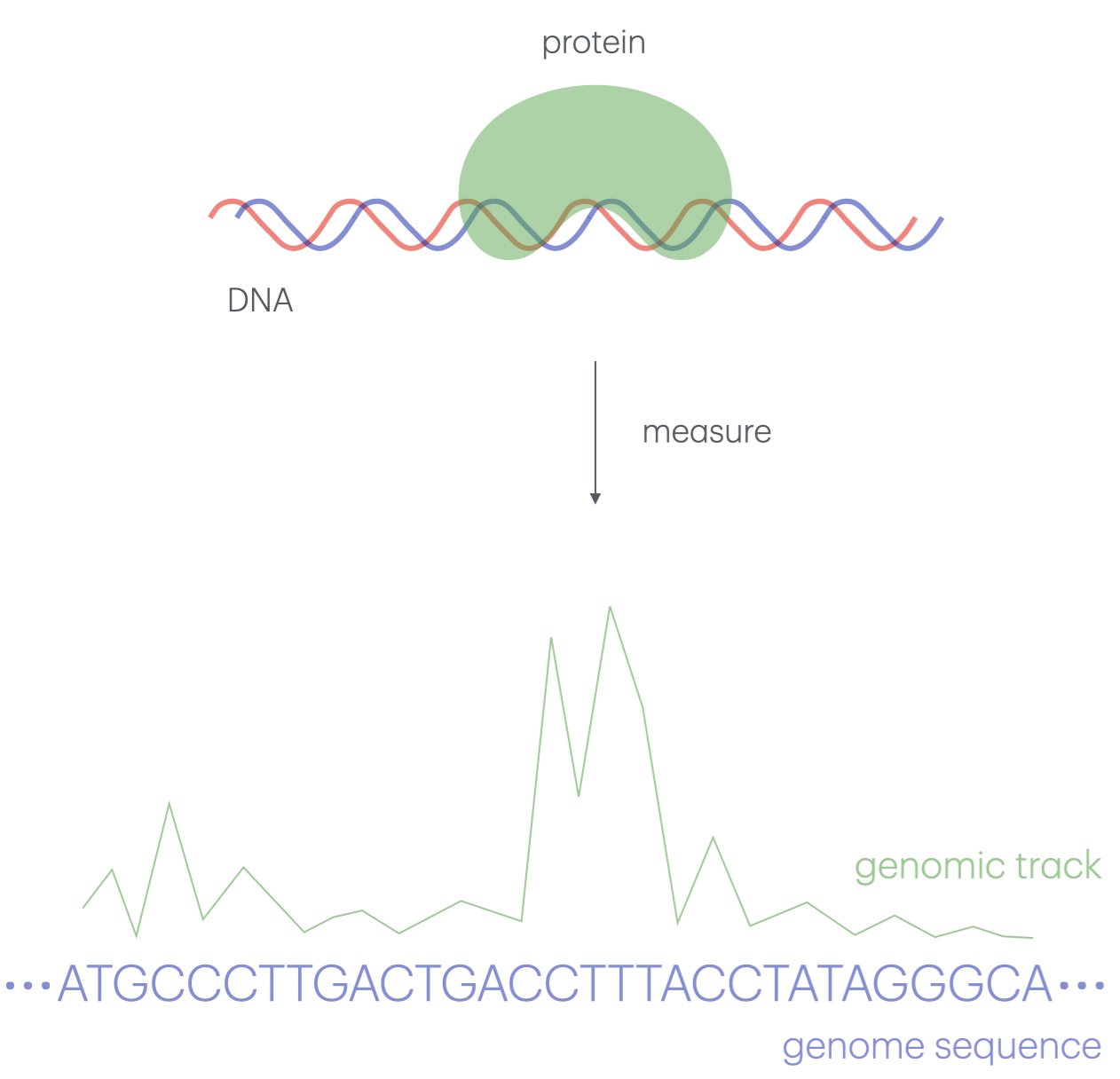}
		\caption{} \label{fig:track-viz}
\end{subfigure}
\begin{subfigure}[t]{0.4\textwidth}
		\centering
		\includegraphics[width=0.8\textwidth]{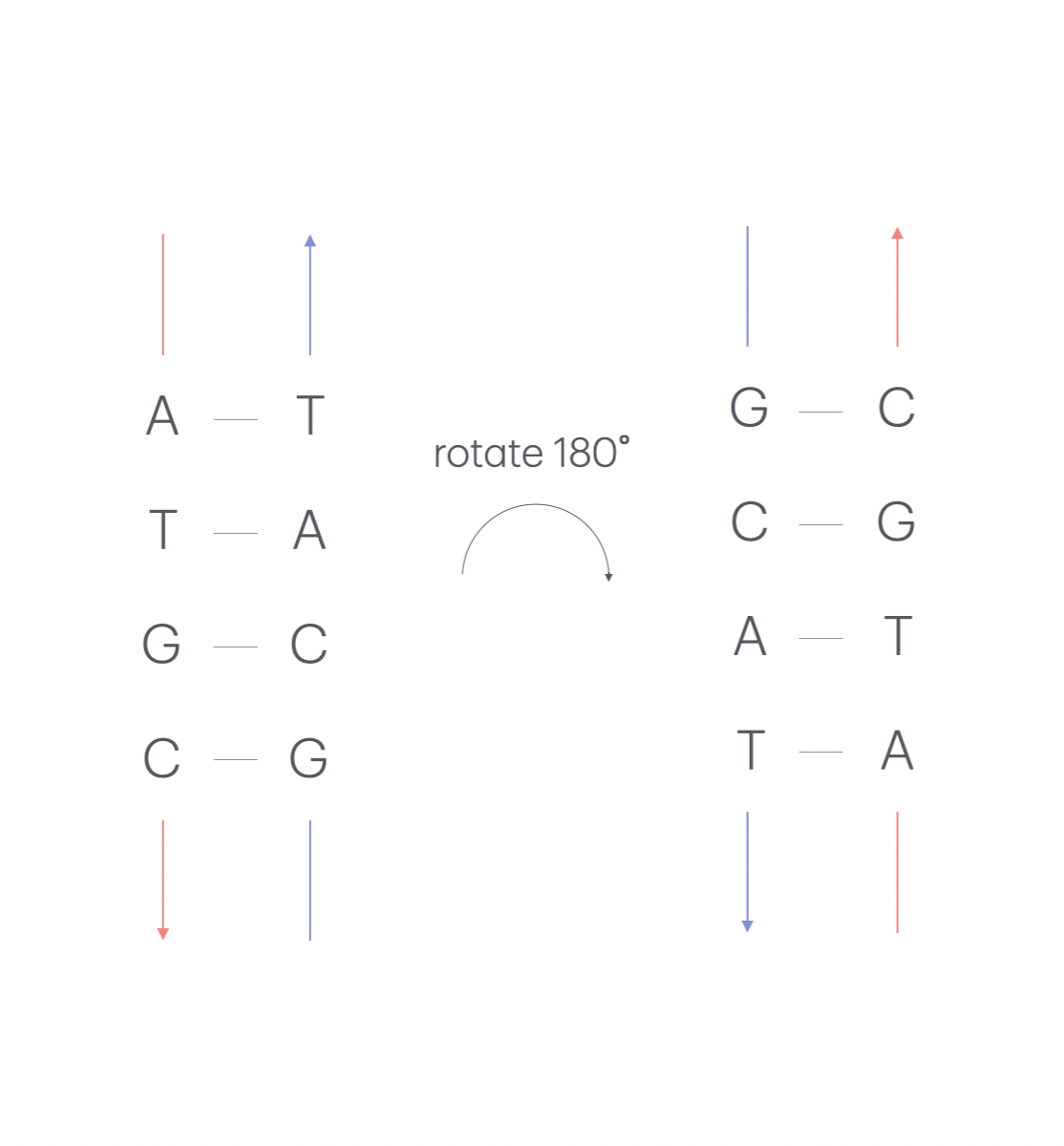}
		\caption{} \label{fig:rc-viz}
\end{subfigure}
\caption{\textbf{Functional genomics data.} (a) Molecular technologies record biological activity at different locations in the human genome, such as chromatin binding or transcription factor binding. The resulting data is a \textit{genomic track}. (b) DNA is a double helix, so a sequence such as ATGC is always paired with its \textit{reverse complement} GCAT on the second strand. Whether we write the sequence as ATGC or as GCAT, they are the same molecule, rotated $180^\circ$.}
\end{figure}

We hope to draw inferences from observational data collected at different points along the genome. This type of data is called a \textit{genomic track} (\Cref{fig:track-viz}).
It is distinct from that used in e.g. genome-wide association studies (GWAS), where there is one scalar outcome per genome \citep{Uffelmann2021-fr,Walker2022-lo}.
Advances in molecular technology enable more fine-grained measurements along a single genome, disaggregating organism-level measurements \citep{Johnson2007-cf,Buenrostro2013-gy}.
We hope to draw causal inferences from the patterns of biological activity these technologies reveal.

We introduce a GCM and causal estimand for functional genomics, then introduce candidate estimators.
The first ends up essentially equivalent to existing methods used by deep learning models of genomic track data. Our derivation from a GCM provides theoretical understanding of how and why these existing methods work.
The second estimation method is distinct, and derived by applying ideas from heterogenous treatment effect estimation to the GCM.
Our experiments on semisynthetic data suggest it can produce improved estimates of variant's effects.

\subsection{A GCM for Functional Genomics}

\paragraph{Model.} We introduce a GCM for functional genomic data (\Cref{fig:genomic-gcm}). The treatment consists of a genome sequence $\ab$. At each position $\omega \in \Omega = \Zb$ is a nucleotide, $a_\omega = \{A,T,G,C\}$. 
The outcome $\yb$ is a measurement of activity along the genome, such as chromatin binding, epigenetic modification, or gene expression. At each position is a scalar, $y_\omega \in \Rb$. 
We are interested in the impact that changes in the genome sequence have on activity, to understand, diagnose or treat human disease \citep{Martin-Geary2025-eb,Avsec2026-tq,Dadush2024-ss}.

To draw causal inferences from a single genome, we consider the underlying symmetries \citep{Mallet2021-wt}.
First, we assume \textit{translation} equivariance: the causal mechanisms are unchanged across different locations in the genome. 
The whole genome is in the nucleus, exposed to essentially the same chemical environment, and the laws of chemistry do not change with location.
Second, we consider \textit{reverse complement} equivariance: the causal mechanisms are unchanged with respect to the orientation of the DNA (\Cref{fig:rc-viz}) \citep{Shrikumar2017-vm,Zhou2022-od}.
This symmetry stems from the double helix structure of the DNA molecule. A strand of DNA with the sequence ATGC always comes with its pair GCAT oriented in the opposite direction, since A forms hydrogen bonds with T and G with C. So a molecule of double stranded DNA with the sequence ATGC is atomically identical to one with GCAT, and the laws of chemistry do not treat them any differently.

Formally, the full symmetry group is denoted by the product $\Gb= \Zb \rtimes \Zb_2$,
 and has elements $\phi_{\tau,s}$ indexed by $\tau \in \Zb$ and $s \in \{-1, +1\}$ \citep{Mallet2021-wt}. 
 Each element acts on $\ab$ as $\phi_{\tau,s}(\ab)_\omega = \sigma_s(a_{s(\omega - \tau)})$, where $\sigma_{+1}$ is the identity and $\sigma_{-1}$ returns each base's complement: $\sigma_{-1}(A) = T, \sigma_{-1}(T) = A, \sigma_{-1}(G) = C, \sigma_{-1}(C) = G$.
The action of $\phi_{\tau,s}$ on $\yb$ is translation, $\phi_{\tau,-1}(\yb)_\omega = \phi_{\tau,+1}(\yb)_\omega = y_{\omega -\tau}$, where we assume the activity data is \textit{unstranded}, i.e. not specific to one or the other DNA strand. 
We assume the causal mechanisms generating functional genomics data are symmetric under $\Gb$.

\paragraph{Estimand.}
Geneticists are interested in the impact of changes in the human genome. 
Given the tiny differences among human genomes, it is standard to ask about the effects of changing only a single letter, i.e. a \textit{single nucleotide polymorphism}.
We consider the conditional average treatment effect (CATE),
\begin{equation}
	\textsc{CATE}_{\wt \to \star}(\ab_{-0}) \triangleq \Eb[Y_{\Nc(0)} \mid \rmdo(a_0 = a_\star), \ab_{-0}] - \Eb[Y_{\Nc(0)} \mid \rmdo(a_0 = a_\wt), \ab_{-0}],
\end{equation}
where $\ab_{-0}$ is a DNA sequence minus the zeroth position, $(\ldots, a_{-2}, a_{-1}, a_{1}, a_2,\ldots)$, and $\Nc(0) = \{\omega: |\omega| < \delta\}$ is a neighborhood around zero.
This effect measures the change in average outcome when we replace the \textit{reference} value $a_\wt$ with the \textit{variant} $a_\star$ in the context of the DNA sequence $\ab_{-0}$.

Reading off \Cref{fig:genomic-gcm}, the causal effect is identified from the observational distribution as,
\begin{equation} \label{eqn:cate-genomic}
	\textsc{CATE}_{\wt \to \star}(\ab_{-0}) = \Eb[Y_{\Nc(0)} \mid a_0 = a_\star, \ab_{-0}] - \Eb[Y_{\Nc(0)} \mid a_0 = a_\wt, \ab_{-0}].
\end{equation}
Formally, the GCM symmetry group $\Zb \rtimes \Zb_2$ satisfies the assumptions of \Cref{thm:estimate_id}, since translations are a subgroup. Assuming finite interference, so the outcome $Y_{\Nc(0)}$ cannot depend on values of $\omega$ arbitrarily far away along the genome, \Cref{thm:estimate_id} will be  satisfied.

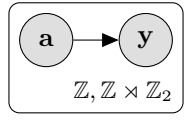
\begin{figure}[t]
\centering
\begin{tikzpicture}

\node[obs]                               (y) {$\yb$};
  \node[obs, left=.6cm of y] (a) {$\ab$};

\edge {a} {y} ;

\plate{in} {(y)(a)} {$\Zb, \Zb \rtimes \Zb_2$};

\end{tikzpicture}
\caption{\textbf{A GCM for functional genomics}. The treatment $\ab$ is the genome sequence, and the outcome $\yb$ is a genomic track recording biological activity. The index set $\Zb$ is discrete nucleotide positions in the genome, and the symmetry is the group of discrete 1D translations and reverse complements.} \label{fig:genomic-gcm}
\end{figure}

\subsection{Estimating Variant Effects}
\paragraph{Estimates from outcome models.} Given finite data, $\{a_{1:n}, y_{1:n}\}$, we can estimate the treatment effect as follows. First, we fit an equivariant model $\hat{\mu}(\ab) \approx \Eb[Y_{\Nc(0)} \mid \ab]$ that predicts $\yb$ from $\ab$. Then, we compute 
\begin{equation} \label{eqn:cate-out}
	\textsc{CATE}_{\wt \to \star}(\ab_{-0}) \approx \hat{\mu}(a_0 = a_\star, \ab_{-0}) - \hat{\mu}(a_0 = a_\wt, \ab_{-0}).
\end{equation}

This causal estimator describes existing algorithms.
Deep functional genomics models such as AlphaGenome, Borzoi, Enformer, SpliceAI or BPNet are not based on explicitly causal or statistical models, assumptions or estimands \citep{Jaganathan2019-yn,Avsec2021-xm,Avsec2021-gj,Linder2025-yq,Avsec2026-tq}. However, they implement \Cref{eqn:cate-out} algorithmically.
Neural networks are trained to predict one or many different outcomes $\yb$ (genomic tracks) given a reference genome $\ab$, based on a sliding context window. They use translation equivariant architectures such as convolutions or transformers, and impose exact or approximate reverse complement equivariance through architecture constraints or data augmentation.
To predict the effects of a variant, they feed in the variant and its surrounding sequence context, and compare the predicted activity to that of a reference (\Cref{eqn:cate-out}).

These methods predict the effects of interventions, generalizing from observations to interventions, according to multiple lines of clinical and laboratory evidence \citep{Landrum2018-jt,Avsec2026-tq}.
For example, models trained on genomic track data from a single reference genome provide variant effect estimates that predict disease in patients with unseen variants, and design synthetic DNA sequences that produce a specific outcome in laboratory experiments \citep{Martin-Geary2025-eb,Brixi2026-sp}.

GCMs provide a theoretical justification and explanation for these methods' success, in terms of underlying causal and statistical assumptions. They also open directions for extension and improvement. We explore one next.

\paragraph{Estimates with propensity models.} Causal estimates are often improved by incorporating a propensity model, not just an outcome model. 
Technically, the approach employed in \Cref{eqn:cate-out}, and in existing deep learning methods, are GCM analogues of the \textit{S-learner} CATE meta-algorithm: train an outcome model, and then feed through the treatment and control \citep{Kunzel2019-nc}. 
We next consider a GCM analogue of the \textit{R-learner}, which incorporates a propensity model to reduce the estimator's sensitivity to modeling errors and regularization bias \citep{Nie2021-ae}.

The R-learner is based on the \textit{Robinson decomposition}, a reparameterization of a causal model. The purpose of this reparameterization is to reduce estimates' variance and bias \citep{Nie2021-ae}.
We can apply the Robinson decomposition to the GCM (\Cref{apx:robinson}). We focus, for simplicity, on modeling the outcome at one point, $Y_{\Nc(\omega)}=Y_\omega$.  We obtain the probabilistic model,
\begin{align} 
a_\omega &\sim \mathrm{Categorical}(e_\eta(\ab_{-\omega})) \label{eqn:r-propensity}\\
	y_\omega &\sim \mathrm{Normal}(m_\phi(\ab_{-\omega}) + (a_\omega - e_\eta(\ab_{-\omega})) \cdot \tau_\theta(\ab_{-\omega}), \sigma), \label{eqn:r-outcome}
\end{align}
where the nucleotide $a_\omega$ is represented as a one-hot encoded vector, and $e_\eta$, $m_\phi$, and $\tau_\theta$ are parameterized functions. 
$\tau_\theta(\ab_{-\omega})$ models the treatment effect, $\textsc{CATE}_{\wt \to a}(\ab_{-0})$. It outputs a vector with an entry for each nucleotide variant $a$. 
$m_\phi(\ab_{-\omega})$ models the conditional mean outcome $\Eb[Y_\omega \mid \ab_{-\omega}]$, describing the baseline activity level expected from the genomic context.
It outputs a scalar.
$e_\eta(\ab_{-\omega})$ models the propensity $\pr(a_\omega \mid \ab_{-\omega})$, outputting a vector describing the probability of each nucleotide occurring at position $\omega$.
Using this model, we can learn by fitting $\eta,\phi$ and $\theta$ to data, and returning $\tau_{\hat\theta}(\ab_{-0}) \approx \textsc{CATE}_{\wt \to \star}(\ab_{-0})$ as an estimate of the causal effect.

The propensity $e_\eta$ models the probability of nucleotides given their context, which is precisely what is learned by \textit{masked DNA language models} \citep{Benegas2023-dy,Sanabria2023-aw,Sanabria2024-qg,Dalla-Torre2025-vf,Nair2025-tt,Shu2026-bh}.
These deep learning models are trained to predict one or more randomly hidden nucleotides, based on the surrounding sequence context \citep{Devlin2018-ji}.
Masking nucleotide $\omega$, they output an estimate of $\pr(a_\omega = a \mid \ab_{-\omega})$.
Unmasked, generative DNA models, such as autoregressive or diffusion models, can also provide estimates of this probability \citep{Amin2021-yk,Nguyen2023-dp,Brixi2026-sp}. None of these DNA language models were originally motivated by causal inference,  
but our GCM R-learner suggests they may be used to improve causal effect estimates by correcting for a variant's propensity. (Note this is distinct from other ways of using DNA language models for effect prediction, which do not employ any $\yb$ data whatsoever, and instead rely on patterns in genome evolution \citep{Benegas2023-dy,Hopf2017-jm,Robson2025-iq}.) 

\subsection{Semisynthetic Experiments}

\begin{table}
\begin{tabular}{l c c c c c}
 & CM & GCM S & GCM S RC & GCM R & GCM R RC \\ \hline \hline
A: effect MSE & 2.633 (0.002) & 2.083 (0.095) & 1.926 (0.076) & \textbf{1.637 (0.021)} & 2.509 (0.014)\\
\hline
B: effect MSE & 0.191 (0.001) & 0.140 (0.011) & \textbf{0.084 (0.011)} & 0.104 (0.010) & 0.184 (0.001) \\
\end{tabular}
\caption{\textbf{Variant effect estimation on semisynthetic genomic data.} Mean squared error of the CATE estimate, across all positions and variants. Estimators: CM: standard causal model without interference across positions; GCM S RC: genomic GCM using an outcome model alone and reverse complement equivariance; GCM S: without reverse complement equivariance; GCM R RC: genomic GCM using a propensity model and reverse complement equivariance; GCM R:  without reverse complement equivariance. We consider two data generating scenarios, A and B. } \label{tbl:genomic}
\end{table}

\paragraph{Setup} We investigate the behavior of different effect estimators empirically, using semisynthetic data.
For $\ab$, we take a section of the human reference genome, hg38, the first thousand nucleotides of the \textit{TERT} gene. This gene encodes telomerase reverse transcriptase, a protein involved in telomere maintenance and aging \citep{Shim2024-gb}.
We generate $\yb$ synthetically. 
Motivated by the biophysics of protein-DNA interactions, we consider a nonlinear mechanism involving a central sequence motif whose effects are modulated by the broader sequence environment, the GC content \citep{De_Boer2019-cs}:
\begin{align}
	\g(a_{\Nc(0)}) &=  \Ib(a_{-1:2} = \mathrm{GTGC}) \exp(1 + \frac{1}{17}\sum_{j=-8}^8 \Ib(a_{-j} \in \{\mathrm{G}, \mathrm{C}\})), \label{eqn:genomic-test-f}\\
y_\omega &\sim \mathrm{Normal}(\g(a_{\Nc(\omega)}) + \g(\sigma_{-1}(a_{\Nc(\omega)})), \sigma).
\end{align}

We implement different estimators for the CATE. First, we consider a conventional causal model that ignores interference across positions, a linear regression predicting $y_\omega$ from $a_\omega$ (CM).
Next we consider estimates using the GCM S-learner strategy, fitting just an outcome model $y_\omega \sim \mathrm{Normal}(\mu_\theta(a_{\Nc(\omega)}), \sigma)$. 
We consider two variations, one where we only require the outcome model to be translation equivariant (GCM S) and one where we also enforce reverse complement equivariance (GCM S RC).
Finally we consider estimates using the GCM R-learner strategy, incorporating a propensity model (\Cref{eqn:r-propensity,eqn:r-outcome}). 
We consider variations where we require the conditional mean outcome, propensity, and treatment effect models to be translation equivariant (GCM R), and where we additionally enforce reverse complement equivariance (GCM R RC).

We parameterize each of the GCM estimators using neural networks with a single convolutional layer. 
We parameterize RC equivariant models by \textit{conjoining}, taking the average of (a) the output of the model applied to a sliding window, and (b) the output applied to the window's reverse complement \citep{Zhou2022-od}.
We use an outcome model that is misspecified with respect to true outcome, to evaluate estimator performance in the presence of complex biological mechanisms. It has a convolutional kernel of size 3 and neighborhood $\Nc(\omega)$ of size 11, smaller than the four letter motif and neighborhood of size 17 in \Cref{eqn:genomic-test-f}.
We use Bayesian inference for each model, following previous work showing semiparametric reparameterizations like the Robinson decomposition are effective for Bayesian as well as frequentist causal inference \citep{Walker2026-mm,Antonelli2018-fu}.
The models are implemented in NumPyro, and we use automated stochastic variational inference \citep{Phan2019-br,Bingham2019-aa,Kucukelbir2017-ir}.
We evaluate each model's performance at estimating the true $\textsc{CATE}_{T \to a}(\ab_{-\omega})$ at each position $\omega$ in the genome sequence, calculating the mean squared error of the posterior mean treatment effect estimate (the \textit{precision in estimation of heterogenous effects}) \citep{Hill2011-ds}.
We repeated the experiments 20 times, including two different noise levels, $\sigma \in \{0.1, 0.3\}$.

\paragraph{Results} We find the geometric models produce improved causal effect estimates, as compared to a standard causal model that ignores interference across positions (\Cref{tbl:genomic}, line 1).
Moreover, incorporating a propensity score correction via the Robinson decomposition leads to  better performance than using an outcome model alone (GCM R versus GCM S and GCM S RC, line 1).
In this scenario, enforcing reverse complement equivariance in addition to translation equivariance does not necessarily improve effect estimates.
While the outcome modeling approach improves with the additional symmetry, the reparameterized model does not.
This suggests there can be an interaction between causal strategies for improving effect estimation (reparameterization) and geometric strategies (adding symmetries).

Propensity modeling techniques offer particular advantages when the outcome model is misspecified. 
We reran the experiments with an alternative outcome generating process using smaller scale sequence features that the outcome model can capture more easily, replacing \Cref{eqn:genomic-test-f} with,
\begin{equation}
	\g(a_{\Nc(0)}) = \Ib(a_{-1:1} = \mathrm{GTG})(1 +  \Ib(a_{5:7} \in \{\mathrm{ACT}, \mathrm{TCG}, \mathrm{TTG}\})). \label{eqn:genomic-test-f-2}
\end{equation}
In this scenario, outcome modeling outperformed the reparameterized model at causal effect estimation, though only when also enforcing reverse complement equivariance (\Cref{tbl:genomic}, line 2).
Without reverse complement equivariance, the reparameterized GCM performed best (GCM S versus GCM R, line 2).

\paragraph{Limitations and directions} This semisynthetic example is small compared to modern deep learning models for human genomics, which train models with hundreds of millions of parameters on billions of nucleotides and thousands of different genomic tracks \citep{Avsec2026-tq}. 
As in previous studies of CATE estimation, our results do not point to a definitive advantage from propensity correction across all data generating scenarios, but rather advantages in specific settings, such as the presence of complex outcome mechanisms \citep{Nie2021-ae}.
Future work may explore methods to implement propensity corrections at large scale using e.g. DNA language models. 
Although the GCM R-learner does not require adding parameters to the outcome model, it does require retraining the outcome model, since the propensity estimates are used during learning.

GCMs provide a broad bridge to import causal methods and ideas into functional genomics.
One direction is to model autocorrelation in the outcome and not just interference, using per-position latent variables $\epsb^y$. This could improve variant effect predictor's calibration. Another direction is to consider different causal graphs. There may be confounders, such as evolutionary pressures, that affect both the genome and the outcomes \citep{Vilhjalmsson2013-ti,Veller2024-az}.
Or, the genome could be used as an instrument, to understand the effect that one genomic track has on another, e.g. chromatin binding on gene expression \citep{Malina2022-ip}.
Finally, causal representation learning methods might be extended to GCMs to improve our understanding and interpretation of genomes and their function \citep{Scholkopf2021-mc,Li2025-kv}.

\section{Discussion}\label{sec:discussion}

We proposed geometric causal models to draw causal inferences from non-i.i.d. data. Geometric causal models replace the usual isolated causal mechanisms with equivariant mechanisms, which give rise to dependence across units. We developed methods to identify causal effects in GCMs, with different causal graphs and groups, without parametric assumptions on the causal mechanisms. 
We developed estimation methods based on geometric deep learning and scalable Bayesian inference.
These methods enable causal inference methods for new kinds of scientific data, such as functional genomics.

Symmetry is a fundamental theme in mathematical sciences. 
This paper investigated how symmetry can enable scientists to learn about cause and effect. 
The results show how symmetry in a data generating process enables generalization from observations to interventions. 
Besides symmetry, to achieve nonparametric identification we relied on an \textit{infinite} world, i.e. an infinite domain $\Omega$ and an infinite group $\Gb$ that could mix the elements of the domain.
To achieve positivity, we needed a notion of \textit{locality}, where the outcome at each point was driven by finite inputs. 
Overall, the blend of symmetry, an infinite world and locality carried sufficient axiomatic structure for causal inference.

\paragraph{Limitations and assumptions} We sought nonparametric identification and estimation results. 
This generality came at a price.

First, our identification theory holds only for infinite domains $\Omega$ and groups $\Gb$.
But many domains are finite or compact, such as data on a sphere.
GCMs are well-defined on these domains, e.g., on the sphere we can consider mechanisms that are equivariant with respect to rotations. 
But we cannot expect nonparametric identification to hold, since we cannot collect infinite observations. 
Nonetheless, identification could still hold under parametric or semiparametric assumptions that restrict the causal mechanisms.
We may also hope to establish partial or near-identification, since as the size of a domain grows, the GCM may approach ergodicity, in the way that finite exchangeable sequences approach i.i.d. sequences~\citep{Diaconis1980-kj}.

Second, to enable positivity or overlap, we assumed finite interference. 
This condition could be relaxed to allow for e.g. smooth decay in space or among neighbors in a social network, by adding additional parametric or semiparametric assumptions on the mechanisms \citep{Li2022-iw}.

Third, to enable generalization to unseen units, we employed ergodicity. Ergodicity is stronger than some design-based assumptions that have previously been used in non-i.i.d. causal inference~\citep{Bojinov2019-nw,Ogburn2022-iv,Wang2025-in,Cristali2022-cv,Basse2017-je,Savje2021-zy}.
The benefit is it helps enables effect estimates for interventions in unobserved regions of space, time or a network~\citep{Ogburn2022-iv,Cristali2022-cv}.

Another limitation of our approach is that it assumes both the symmetry group and causal graph are known. 
Future work could extend methods for discovering and testing symmetries in data to GCMs, and could extend methods for discovering and testing causal graphs to GCMs \citep{Chiu2023-zs,Chen2026-pq,Zanga2025-yf}.

\paragraph{Outlook} 
GCMs provide a flexible modeling language for causal inference from structured scientific data. 
We have focused on groups involving shifts and permutations, but many other groups are encountered across science, such as crystallographic groups in materials science or Lorentz groups in physics \citep{Adams2023-bu,Chiu2024-og}.
Large scale data collection efforts and machine learning systems are increasingly developed and deployed for scientific applications with diverse symmetries \citep{Kohler2020-sx,Watson2023-sp,Zeni2025-ol,Cornet2025-jv,Kabylda2025-of}.
 Future work on GCMs may use such datasets and models to address causal questions, including correcting for bias, analyzing counterfactuals, uncovering representations or generalizing to new environments.
 
GCMs' broad scope may also help address increasing demand for customized, reliable and automated scientific data analysis and experiment planning.
Probabilistic programming and program synthesis, including natural language-based methods, provide increasingly automated tools for model design, inference, criticism, improvement and deployment \citep{Bingham2019-aa,Karwowski2026-md}.
However, causal inference from structured data has largely remained the domain of bespoke tools and theory.
GCMs provide a unified model design language and basic theoretical guarantees that could help expand the reach of probabilistic programming deeper into causal inference \citep{Tavares2021-dw,Agrawal2024-ke,Hoyt2025-ba,Luedtke2026-eo}.
Future work could explore new tools for implementing, criticizing and iterating on geometric causal models.

\bibliographystyle{abbrvnat}
\bibliography{references}

\appendix
\setcounter{figure}{0}
\renewcommand{\thefigure}{S\arabic{figure}}

\section{Marginalizing GCMs}\label{apx:marginalize}

An important property of GSCMs is that they satisfy the same rules for marginalizing out causal variables as conventional SCMs. 

In particular, consider a causal variable $\yb$ with a single child $\zb$,
\begin{align}
	\epsb^y \sim \pr(\epsb^y) \quad \quad \yb &= \f^y(\xb^{\pa(v)}, \epsb^y),\\
	\epsb^z \sim \pr(\epsb^z) \quad \quad \zb &= \f^z(\yb, \epsb^z)
\end{align} 
We can marginalize out $\yb$ and still obtain a valid GSCM. In particular, we obtain
\begin{align}
	\tilde \epsb^z \sim \pr(\tilde \epsb^z)\quad \quad & \zb \sim \tilde \f^z (\xb^{\pa(v)}, \tilde \epsb^z) \\
	\text{ where } \pr(\tilde \epsb^z) &\triangleq \pr(\epsb^y)\pr(\epsb^z)\\
	\tilde \f^z (\xb^{\pa(v)}, \tilde \epsb^z) &\triangleq \f^z(\f^y(\xb^{\pa(v)}, \epsb^y), \epsb^z)
\end{align}
Since the product of independent invariant distributions is i.i.d., $\pr(\tilde \epsb^z)$ is i.i.d.. Since the composition of equivariant functions is equivariant, $\tilde \f^z$ is equivariant.
So in GSCMs, like in SCMs, we can marginalize out causal variables with a single child. 

Also, as in SCMs, we cannot marginalize out confounders, i.e. causal variables with more than one child, since this gives rise to correlation between the children that is not described by the model.

\section{Proofs}

\subsection{Proof of \Cref{prop:stochastic_mech_equiv}}\label{apx:conditional-equiv}
\begin{proof}
	$\pr(\epsb^v)$ is i.i.d. across $\omega \in \Omega$ (\Cref{def:GSCM}), and the group action $\phi:\Omega \to \Omega$ is a bijection (\Cref{def:transform_group}), so $\pr(\epsb^v)=\prod_{\omega \in \Omega}\pr(\epsb^v_\omega)=\prod_{\omega \in \Omega}\pr(\epsb^v_{\phi^{-1}(\omega)})=\pr(\phi(\epsb^v))$ is invariant. Now,
\begin{align}
	\Prob(\xb^v \in \Xi \mid \xb^{\pa(v)}) &= \Prob(\xb^v \in \Xi \mid \phi^{-1}(\phi(\xb^{\pa(v)})))\\
	&=\int \Ib(\f^v(\phi^{-1}(\phi(\xb^{\pa(v)})), \epsb^v) \in \Xi) \pr(\epsb^v)\di \epsb^v \\
	&= \int \Ib(\f^v(\phi^{-1}(\phi(\xb^{\pa(v)})), \phi^{-1}(\epsb^v)) \in \Xi) \pr(\epsb^v)\di \epsb^v \\
	&= \int \Ib(\phi^{-1}(\f^v(\phi(\xb^{\pa(v)}), \epsb^v)) \in \Xi) \pr(\epsb^v)\di \epsb^v\\
	&= \int \Ib(\f^v(\phi(\xb^{\pa(v)}), \epsb^v) \in \phi(\Xi)) \pr(\epsb^v)\di \epsb^v\\
	&= \Prob(\xb^v \in \phi(\Xi) \mid \phi(\xb^{\pa(v)}))
\end{align}
where for the third line we use the fact that the noise distribution is invariant, and for the fourth line we use the fact that the mechanism $\f^v$ is equivariant.
\end{proof}
\subsection{Proof of \Cref{thm:gcm_ergodic}} \label{apx:proof_ergodic}

The first step of the proof is to show that the exogenous noise distribution in a GCM is ergodic. 
This follows from an analogous argument to that used to show i.i.d. variables are ergodic under the shift group, but now generalizing the domain $\Omega$ and the group $\Gb$ \citep[Chap. 1]{Sarig2023-bz}.
Specifically, we employ Thm. 1 of \citep{Ayach2025-hu}, which generalizes zero-one laws to \textit{positional symmetries}, which correspond in our context to index transforms. 

Let $\pr(\epsb)$ denote the i.i.d. and invariant distribution of the noise, where $\epsilon_\omega \in  \Ec$. Let $B \subseteq \Ec^{\Omega}$ denote a set that is invariant to $\Gb$.
The mixing condition (\Cref{asm:mix}) implies that for any finite $S \subset \Omega$ there is an index transform $\phi_n \in \Gb$ for which $|\phi_n(S) \cap S| = 0$.
Thm. 1 of \citep{Ayach2025-hu} says that this implies $\Pr(\epsb \in B) \in \{0, 1\}$. Since this holds for any invariant set $B$, $\pr(\epsb)$ is ergodic (\Cref{def:ergodic}). 

We now show GCMs are ergodic.
First, since i.i.d. and invariant distributions are ergodic, the joint distribution $\pr(\epsb^\Vc)$ over noise variables is ergodic.
Now, let $\F^\Cc: (\Ec^\Vc)^\Omega \to (\Xc^\Cc)^\Omega$ denote the mapping from exogenous noise variables to endogenous variables $\xb^\Cc$, defined by the GSCM.
That is, $\xb^\Cc = \F^\Cc(\epsb^\Cc)$. 
Since $\F^\Cc$ is a composition of equivariant functions, it must be equivariant.
So the distribution $\pr(\xb^\Cc)$ is the pushforward of an ergodic distribution through an equivariant map,
\begin{align}
	\epsb^\Vc &\sim \pr(\epsb^\Vc)\\
	\xb^\Cc &= \F^\Cc(\epsb^\Vc),
\end{align}
This implies $\pr(\xb^\Cc)$ is ergodic, by Lemma A1.1 in \cite{Kallenberg2005-gv}.
\qed	

\subsection{Proof of \Cref{thm:docalculus}} \label{apx:proof_do}
\begin{proof}
As with conventional causal models, GCMs follow a Markov decomposition according to their causal graph (\Cref{def:GCGM}) :
\begin{equation}
		\pr(\xb^\Vc) = \prod_{v \in \Vc} \pr(\xb^v \mid \xb^{\pa(v)}).
\end{equation}
Thus, we can apply Thm. 5 of \citep{Shpitser2006-jg}.
\end{proof}

\subsection{Proof of \Cref{thm:docalculus_complete}} \label{apx:proof_do_complete}
\begin{proof}
	We need to show that if the effect $\pr(\yb \s \rmdo(\ab = \ab_\star))$ is not identified by do-calculus, it cannot be computed as a functional of $\pr(\xb^{\Vc_\obs})$. That is, there does not exist a functional $\psi(\cdot)$ such that 
	\begin{equation} \label{eqn:functional-do}
		\psi(\pr(\xb^{\Vc_\obs})) = \pr(\yb \s \rmdo(\ab_\star))
	\end{equation} 
	for all $\pr(\xb^\Vc) \in \Pc_{\Gb,\Gc}$, where $\Pc_{\Gb,\Gc}$ denotes the set of all joint distribution that can be produced by a GCM with causal graph $\Gc$ and group $\Gb$.
	
	Let $\Pc_{\Gc}$ denote the set of all joint distributions that can be produced by a conventional causal model with a causal graph $\Gc$. This is the set of joint distributions produced by conditionally i.i.d. causal mechanisms, that is in the notation of GCMs, distributions of the form,
	\begin{equation} \label{eqn:conditional-iid}
		\pr(\xb^v \mid \xb^{\pa(v)}) = \prod_{\omega \in \Omega} \pr(x_\omega^v \mid x_\omega^{\pa(v)}),
	\end{equation}
	for countably infinite $\Omega$.
	This conditional distribution is equivariant under the group of permutations of $\Omega$. 
	Since any index transform $\phi: \Omega \to \Omega$ is a permutation, \Cref{eqn:conditional-iid} must also be equivariant under any group $\Gb$ of index transforms. So $\Pc_{\Gc} \subseteq \Pc_{\Gb, \Gc}$, i.e. the set of joint distributions produced by conventional causal models is a subset of the set of joint distributions produced by GCMs, for any group $\Gb$ that only contains index transforms.
	
	Thm. 6 of \citep{Shpitser2006-jg} shows do-calculus is complete for conventional causal models, meaning that if the effect is not identified by do-calculus, there does not exist a $\psi$ such that \Cref{eqn:functional-do} holds for all $\pr(\xb^\Vc) \in \Pc_{\Gc}$.
	Then, since there is no $\psi$ such that \Cref{eqn:functional-do} holds for all elements of $\Pc_{\Gc}$, and $\Pc_{\Gc} \subseteq \Pc_{\Gb, \Gc}$, there cannot be a $\psi$ that holds for all elements of $\Pc_{\Gb, \Gc}$.	
\end{proof}

\subsection{Proof of \Cref{thm:iv-id}} \label{apx:proof-iv}

\begin{proof}
We first show that the IV model is linear-Gaussian.

To do this, we first show the distribution over each noise variable $\pr(\epsb^v)$ is Gaussian with mean zero.
By construction of the GCM, $\pr(\epsb^v)$ is invariant to $\Ob$. It is ergodic by \Cref{asm:mix}, since we can choose a series of transformations $\phi_n \in \Ob$ that swap the first $n$ elements with the second $n$ (the same series as we chose for $\Sb$ in \Cref{sec:ergodic}). 
	This implies $\pr(\epsilon^v)$ must be Gaussian with mean zero, i.e. $\epsilon_\omega^v \overset{iid}{\sim} \Nc(0, \sigma^v)$, by Freedman's theorem~\citep[Theorem 1.31]{Orbanz2015-fm,Kallenberg2005-gv}.
	
We now show the causal mechanisms are linear. 
We first examine univariate mechanisms $\yb = \f(\xb)$.  
Since the permutation group is a subset of the orthogonal group, $\Sb \subset \Ob$, the mechanisms must be permutation equivariant. Finite interference (\Cref{asm:finite_depend})  implies that any mechanism $\yb = \f(\xb)$ can be written
\begin{equation}
	y_\omega = \g(x_\omega)
\end{equation}
for $\omega \in \Nb$. The reason is if modifying any one $x_{\omega'\neq \omega}$ affected $y_\omega$, then $y_\omega$ would depend on an infinite number of other $x_{\omega' \neq \omega}$, by applying permutations that move $\omega'$ while leaving $\omega$ in place. 

Now let $\Zero = (0)_{\omega \in \Nb}$ and note $\phi(\Zero) = \Zero$ for any rotation $\phi \in \Ob$. From equivariance, we then have
\begin{equation}
	\f(\Zero) = \f(\phi(\Zero)) = \phi(\f(\Zero))
\end{equation}
for any $\phi \in \Ob$. 
Choose $\phi$ to be a rotation such that $\phi(\xb)_1 = -x_1$. Then we have $\g(0) = \f(\Zero)_1 = \phi(\f(\Zero))_1 = - \g(0)$. 
Hence, $\g(0) =0$.

Next, let $\xb = (a, 0, 0,\ldots)$ and note $\f(\xb) = (\g(a), 0, 0, \ldots)$. Choose $\phi_\theta \in \Ob(2) \subset \Ob$ to be a rotation of the first two elements,
\begin{equation}
	\phi = \left(\begin{matrix}
		\cos \theta & -\sin \theta & 0 \ldots\\
		\sin \theta & \cos \theta & 0 \ldots\\
		0 & 0 & 1 \ldots\\
		\vdots&\vdots &\ddots 
	\end{matrix}\right).
\end{equation}
Equivariance, $\f(\phi_\theta(\xb)) = \phi_\theta(\f(\xb))$, then implies
\begin{equation}
	\g(a \cos \theta ) = \g(a) \cos \theta,
\end{equation}
for all $a$ and $\theta$. 
So, $\g$ is linear. We thus have, for some $m \in \Rb$,
\begin{equation}
	y_\omega = m x_\omega
\end{equation}
for all $\omega \in \Nb$. 

Next consider multivariate mechanisms, $\yb = \f(\xb, \zb)$, with equivariance $\f(\phi(\xb), \phi(\zb)) = \phi(\f(\xb, \zb))$. 
From permutation equivariance and finite interference, $y_\omega = g(x_\omega, z_\omega)$, by the argument above.
Setting $\xb = \zb = \Zero$, we find $\g(0, 0) = 0$, by the argument above. Setting $\zb = \Zero$ and $\xb = (a, 0, 0, \ldots)$, we obtain  $\g(x, 0) = m^x x$ by the argument above, and likewise $\g(0, z) = m^z z$. 
Now, let $\xb = (a, 0, 0, \ldots)$ and $\zb = (0, b, 0, 0, \ldots)$. Equivariance then implies,
\begin{equation}
	\f(\phi_\theta(\xb), \phi_\theta(\zb))_2 = \g(a \sin \theta, b \cos \theta) = \sin \theta\, \g(a, 0) + \cos \theta\, \g(0, b) = (\sin \theta) a\, m^x + (\cos \theta) b\, m^z,
\end{equation}
for all $\theta$ and all $a,b \in \Rb$. Thus $\g$ is linear in $(x, z)$, and we have, $\g(x, z) = m^x x + m^z z$. 

The argument extends to more than two inputs, giving $g(x_1, \ldots, x_k) = m_1 x_1 + \ldots +m_k x_k$, by applying rotation matrices in higher dimension, $\phi \in \Ob(k)$.

We now know the IV model must take the linear-Gaussian form,
\begin{align}
	z_\omega &\sim \Nc(0, \sigma^z)\\
	u_\omega &\sim \Nc(0, \sigma^u)\\
	a_\omega &\sim \Nc(m^{za} z_\omega + m^{ua} u_\omega, \sigma^a)\\
	y_\omega &\sim \Nc(m^{ay} a_\omega + m^{uy} u_\omega, \sigma^y),
\end{align}
for $\omega \in \Nb$. We can read off that $\mathbb{E}[Y_\omega; \rmdo(\ab_{\star})] = m^{ay}a_{\star,\omega}$. So to identify the average causal effect, it suffices to identify $m^{ay}$ from $\pr(z,a,y)$. Using the fact that the joint distribution $\pr(z,u,a,y)$ is multivariate normal, we can note,
\begin{equation}
	\frac{\Cov_\pr(z,y)}{\Cov_\pr(z,a)} = \frac{m^{ay} m^{za} (\sigma^z)^2}{m^{za} (\sigma^z)^2} = m^{ay}.
\end{equation}
This is the standard linear IV estimator \citep[Chap. 4]{Angrist2009-ah}. 

Finally we confirm that the denominator is non-zero, so the ID formula is well-defined. Since $\mathbb{E}_\pr[A_\omega \s \rmdo(\zb_\star)] = m^{za}\, z_{\star,\omega}$, the relevance condition implies $m^{za} \neq 0$.
Since the instrument is not a constant, $\sigma^z > 0$. Hence $m^{za} (\sigma^z)^2 > 0$.
\end{proof}

\subsection{Proof of \Cref{thm:lindenstrauss}} \label{apx:proof_linden}
\begin{proof}
We first need to show the estimator is well defined, in the sense that it depends only on the value of $\xb$ at a finite subset of $\Omega$, i.e. the right hand side of \Cref{eqn:lindenstrauss} can be computed from finite data $x_{A_n}$. 
First, note $A_n$ must be finite because it is a finite union of finite sets (by \Cref{asm:finite-tempered}). 
Then, we can see that any $\xb'$ that matches $\xb$ on this domain, $x'_{A_n} = x_{A_n}$, produces the same value of the estimate. The reason is that since $\phi^{-1}(S) \subseteq A_n$, we have $x'_{\phi^{-1}(S)} = x_{\phi^{-1}(S)}$ and hence $\phi(\xb)_S = \phi(\xb')_S$.
 
	The result then follows from Theorem 1.1 in \citep{Lindenstrauss2001-sp}.
\end{proof}

\subsection{Proof of \Cref{thm:estimate_id}} \label{apx:proof_estimate_id}
\begin{proof}
Let $\Mc$ denote the GCM (\Cref{def:GSCM}). Our proof proceeds by constructing a conventional causal graphical model $\tilde \Mc$, for which the effect is the same as in the GCM. 
In $\tilde \Mc$, the identification formula from do-calculus is entirely in terms of \emph{finite} marginals, i.e. it is a functional of $\pr(x^{\Vc_\obs}_I)$ for $|I| < \infty$.
We then estimate the terms of the identification formula using Lindenstrauss' ergodic theorem.

Let $I^v \subset \Omega$ denote the subset of the domain for which modifying $x^v_{I^v}$ affects $y_S$. In particular, if $\yb$ has a parent $\ab$, $I^a = \cup_{\omega \in S} D^y_\omega$. More generally, $I^v$ is nonempty for any ancestor $v$ of $y$ in the causal graph, and defined recursively as $I^v = \cup_{v' \in ch(v)} \cup_{\omega \in I^{v'}} D^{v'}_\omega$, where $ch(v)  \subset \Vc$ denotes the nodes in the causal graph that are children of $v$ and ancestors of $y$, and with the initialization $I^y = S$.
Then, from, \Cref{asm:finite_depend}, $I^v$ must be finite for all $v \in \Vc$.

Consider a conventional causal graphical model $\tilde \Mc$ with the same graph $\Gc$ as the GCM, but with each variable $\xb^v$ replaced by $x^v_{I^v}$, and with the GCGM's mechanisms $\pr(\xb^v \mid \xb^{\pa(v)})$ replaced by their conditionals over the relevant subset of the domain,
\begin{equation}
	x^v_{I^v} \sim \pr(x^v_{I^v} \mid x^{\pa(v)}_{I^{\pa(v)}}).
\end{equation}
Since variables outside $I^{\Vc}$ cannot causally influence $y_S$ (\Cref{asm:finite_depend}), the effect $\pr(y_S \s \rmdo(\ab_\star))$ in the original GCM must be equal to the effect $\pr(y_S \s \rmdo(a_{I^a} = a_{\star,I^a}))$ in this conventional model $\tilde \Mc$.
The graph of $\tilde \Mc$ is the same as $\Mc$, so the causal effect is identified by do-calculus.

Since the effect is identified by do-calculus, we must be able to write it as a continuous functional of conditional distributions of the observational distribution of $\tilde \Mc$, $\pr(x^{\Vc_\obs}_{I^{\Vc_\obs}})$~\citep{Shpitser2006-jg}. So,
\begin{align}
	\pr(y_S \s \rmdo(\ab_{\star})) &= \psi_{\ab_\star}(\pr(x^{\Vc_\obs}_{I^{\Vc_\obs}})) \label{eqn:do-calc-decomp}
\end{align}
for some continuous $\psi_{\ab_\star}$.
The positivity assumption ensures $\pr(x^{\Vc_\obs}_{I^{\Vc_\obs}}) > 0$. The results of \cite{Shpitser2006-jg} show this condition is sufficient to ensure the right hand side of \Cref{eqn:do-calc-decomp} is finite and well-defined.

From \Cref{thm:gcm_ergodic}, $\pr(\xb^{\Vc_\obs})$ must be ergodic. 
We can construct a consistent estimator $q_n(x^{\Vc_\obs}_{A_n})$ of $\pr(x^{\Vc_\obs}_{I^{\Vc_\obs}})$, according to \Cref{thm:lindenstrauss} (\Cref{eqn:lindenstrauss}).
By the continuous mapping theorem, we can plug this into the functional $\psi_{\ab_\star}(\cdot)$ to obtain a consistent estimator of the effect,
\begin{equation}
	q_n(x^{\Vc_\obs}_{A_n} \s \rmdo(\ab_\star)) \triangleq \psi_{\ab_\star}(q_n(x^{\Vc_\obs}_{A_n})) \xrightarrow[n \to \infty]{a.s.} \pr(y_S \s \rmdo( \ab_\star)).
\end{equation}
\end{proof}
\noindent Note the positivity condition in \Cref{thm:estimate_id} is a stronger assumption than strictly necessary, and can be relaxed to only require positivity for variables we condition on~\citep{Shpitser2006-jg}.

\section{Spatial Simulation} \label{apx:spatial-sim}

\begin{figure}
\centering
\begin{subfigure}[t]{0.3\textwidth}
		\centering
		\includegraphics[width=\textwidth]{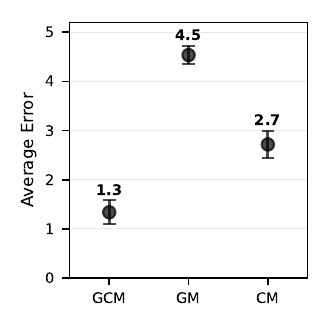}
		\caption{} \label{fig:spatial_gcm_avg_error}
	\end{subfigure}
	\begin{subfigure}[t]{0.3\textwidth}
		\centering
		\includegraphics[width=\textwidth]{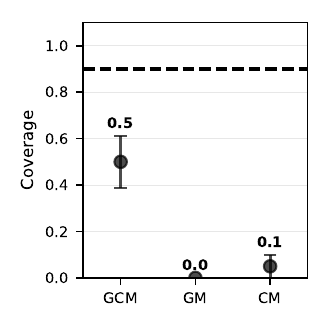}
		\caption{} \label{fig:spatial_gcm_coverage}
	\end{subfigure}	
	\caption{\textbf{Spatial GCM summary} (a) Average error and standard error of the mean across 20 simulations. (b) Fraction of simulations in which the credible interval (5th-95th) covered the true treatment effect, and standard error of the mean. Dashed line indicates expected coverage of 90\%.}
\end{figure}

\paragraph{Data generation.} The data generating process is,
\begin{align}
	\ub &\sim \GPc(\mathbf{0}, \kappa)\\
	\xb &\sim \GPc(100 B \ast \ub, 0.05^2 k)\\
	\ab &\sim \GPc(-100 B \ast \ub, 0.05^2 k)\\
	\yb &\sim \GPc(50 B \ast \ab + 100 B \ast \xb, 0.05^2 k)
\end{align}
where $B_\omega \triangleq \exp(-\frac{1}{2} \left(\frac{\|\omega\|_2}{0.02}\right)^2)$ and $\kappa(\omega, \omega') = \sqrt{\pi} 0.01 \exp\left(-\frac{1}{4} \left(\frac{\|\omega-\omega'\|_2}{0.01}\right)^2\right)\\ = \int \exp\left(-\frac{1}{2}\left(\frac{\|\omega-\tilde \omega\|_2 }{0.01}\right)^2\right) \exp\left(-\frac{1}{2}\left(\frac{\|\omega'-\tilde \omega\|_2 }{0.01}\right)^2\right) d\tilde \omega.$  This process has an interference length scale of $0.02$ and noise autocorrelation length scale of $\sqrt{2} \times 0.01$. 
The treatment effect is, 
\begin{equation}
	\Eb[\yb \s \rmdo(\ab = \mathbf{1})] - \Eb[\yb \s \rmdo(\ab = \mathbf{0})] = 50 \times \sqrt{2 \pi} \times 0.02 = \sqrt{2\pi}
\end{equation}
where $\sqrt{2 \pi} \times 0.02$ comes from the Gaussian integral over $B$.

To produce approximate samples from this process, we sample a set of 2000 measurement points $\omega_1, \ldots, \omega_{2000}$ uniformly at random within the unit square. 
We estimate the value of each mean function at each point, e.g. $(B \ast \ub)_\omega$, by Monte Carlo integration over the unit square using these samples. 
We sample the noise using a multivariate normal with the reparameterization $\mu + A \ast \zb$, where $z$ denotes white noise and $A_\omega = \exp(-\frac{1}{2} (\frac{\|\omega\|_2}{0.01})^2)$, and evaluating $\mu$ and $\zb$ at the same set of measurement points.

\paragraph{Bayesian model.} We consider the Bayesian GCM model,
\begin{align}
	c^{ay} &\sim \mathrm{Normal}(0, 100)\\
	c^{xy} &\sim \mathrm{Normal}(0, 100)\\
	\tau &\sim \mathrm{Normal}(0, 10)\\
	\yb &\sim \GPc(c^{ay} B \ast \ab + c^{xy} B \ast \xb, \tau^2 \kappa),
\end{align}
where $\mathrm{Normal}(\mu, \sigma)$ is a Gaussian with mean $\mu$ and standard deviation $\sigma$.
Note we assumed the spatial length scales are known. The treatment effect is computed from posterior samples as $c^{ay} \times \sqrt{2\pi} \times 0.02$, where $0.02$ is the interference length scale.

As one comparison, we consider a geometric model that does not correct for confounding (GM) defined as,
\begin{align}
	c^{ay} &\sim \mathrm{Normal}(0, 100)\\
	\tau &\sim \mathrm{Normal}(0, 10)\\
	\yb &\sim \GPc(c^{ay} B \ast \ab, \tau^2 \kappa).
\end{align}
The treatment effect has the same expression as above.

As another comparison, we consider a conventional causal model (CM) that does not account for spatial interference or autocorrelation. 
\begin{align}
	c^{ay} &\sim \mathrm{Normal}(0, 100)\\
	c^{xy} &\sim \mathrm{Normal}(0, 100)\\
	\tau &\sim \mathrm{Normal}(0, 10)\\
	y_\omega &\sim \mathrm{Normal}(c^{ay} a_\omega + c^{xy} x_\omega, \tau).
\end{align}
The treatment effect is then $c^{ay}$.

\paragraph{Bayesian inference.} We implement the model in NumPyro \citep{Phan2019-br}. To approximate the GP, we use a multivariate normal likelihood evaluated at the measurement points, and approximate the mean by Monte Carlo integration, as in the data generating process. For numerical stability we add a small constant, $\epsilon=0.01$, to the diagonal of the covariance matrix and compute the model likelihood using the Cholesky decomposition.
We perform inference using the No-U Turn Sampler (NUTS) with a single chain and default parameters, using 100 steps for warmup and drawing 500 samples \citep{Phan2019-br,Bingham2019-aa,Hoffman2014-ic}. We initialize the sampler from an estimate of the maximum a posteriori parameter value, computed by optimizing the log likelihood with 40{,}000 steps of Adam at a step size of 0.00005. 

\paragraph{Simulation configuration.} We create 10 independent datasets by drawing from the data generating process. In each dataset, each variable $\xb,\ab,\yb$ is measured at 2000 points in the unit square. For each, we also created a smaller dataset where we subset the data by retaining only points in the domain $\omega_j \in (0, 0.5)$, i.e. points from a square of area $0.5^2=0.25$. 
We run inference on each of these 20 datasets and report the average error of the posterior mean treatment effect (\Cref{fig:spatial_gcm_avg_error}) and the average calibration over the 20 datasets, i.e. how often the true treatment effect fell in a 90\% credible interval (\Cref{fig:spatial_gcm_coverage}). The credible interval is taken to be between the 5th and 95th percentiles of the posterior.

\section{Array Simulation} \label{apx:array-sim}
\begin{figure}
	\centering
	\begin{subfigure}[t]{0.3\textwidth}
			\centering
			\includegraphics[width=\textwidth]{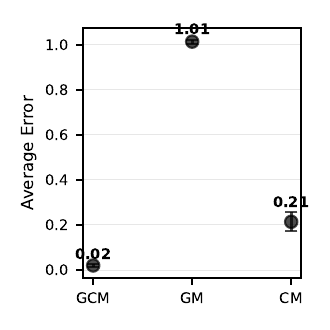}
			\caption{} \label{fig:array_gcm_avg_error}
		\end{subfigure}
		\begin{subfigure}[t]{0.3\textwidth}
			\centering
			\includegraphics[width=\textwidth]{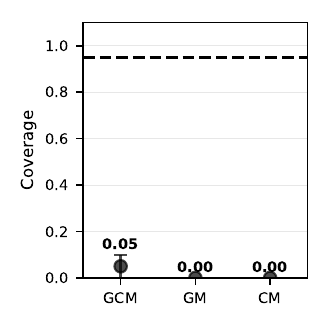}
			\caption{} \label{fig:array_gcm_coverage}
		\end{subfigure}	
		\caption{\textbf{Array GCM summary} (a) Average error and standard error of the mean across 20 simulations. (b) Fraction of simulations in which the credible interval (2.5th-97.5th) covered the true treatment effect, and standard error of the mean. Dashed line indicates expected coverage of 95\%.}
	\end{figure}

\paragraph{Data generation.} The data generating process is
\begin{align}
	u_{\omega} &= \epsilon^u_{\omega_1 \diamond} + \epsilon^u_{\diamond\omega_2} \quad\quad\quad\quad\quad\quad\quad\quad\quad\quad\quad\quad\quad \quad \quad\,\,\, \epsilon^u_{\omega_1 \diamond}, \epsilon^u_{\diamond\omega_2} \sim \mathrm{Normal}(0, 1)\\
	a_\omega &\sim \mathrm{Normal}(- u_\omega + 2 \epsilon^a_{\omega_1\diamond} + 2 \epsilon^a_{\diamond\omega_2}, 0.1) \,\,\,\,\quad\quad\quad \quad \quad \epsilon^a_{\omega_1 \diamond}, \epsilon^a_{\diamond\omega_2} \sim \mathrm{Normal}(0, 1)\\
	x_\omega &\sim \mathrm{Normal}(a_\omega + 2 \epsilon^x_{\omega_1\diamond} + 2 \epsilon^x_{\diamond\omega_2}, 0.1) \,\,\,\quad \quad \quad \quad \quad \quad \epsilon^x_{\omega_1 \diamond}, \epsilon^x_{\diamond\omega_2} \sim \mathrm{Normal}(0, 1)\\
	y_\omega &\sim \mathrm{Normal}(x_\omega + u_\omega + 2 |\epsilon^y_{\omega_1\diamond}| + 2 |\epsilon^y_{\diamond\omega_2}|, 0.1) \quad \quad \quad\epsilon^y_{\omega_1 \diamond}, \epsilon^y_{\diamond\omega_2} \sim \mathrm{Normal}(0, 1)
\end{align}
We draw a dataset with $\omega_1 \in \{1, \ldots, 200\}$ rows and $\omega_2 \in \{1, \ldots, 50\}$ columns.
The treatment effect for this model is,
\begin{equation}
	\Eb[\yb \s \rmdo(\ab = \mathbf{1})] - \Eb[\yb \s \rmdo(\ab = \mathbf{0})] = 1.
\end{equation}

\paragraph{Bayesian model.} We consider the Bayesian GCM model,
\begin{align}
	\gamma &\sim \mathrm{Normal}(0, 100)\\
	x_\omega &\sim \mathrm{Normal}(\gamma a_\omega + \epsilon_{\omega_1\diamond }^x + \epsilon_{\diamond\omega_2}^x, 0.1) \quad \quad \quad \quad \quad \quad \quad \quad \epsilon_{\omega_1\diamond }^x, \epsilon_{\diamond\omega_2}^x \sim \mathrm{Normal}(0, 100)\\
\beta^a &\sim \mathrm{Normal}(0, 100)\\
\beta^x &\sim \mathrm{Normal}(0, 100)\\
	y_\omega &\sim \mathrm{Normal}(\beta^{a} a_\omega + \beta^x x_\omega + \epsilon_{\omega_1\diamond}^y + \epsilon_{\diamond\omega_2}^y, 0.1) \quad \quad \quad \quad \epsilon_{\omega_1\diamond}^y, \epsilon_{\diamond\omega_2}^y \sim \mathrm{Normal}(0, 100)
\end{align}
where $\mathrm{Normal}(\mu, \sigma)$ is a Gaussian with mean $\mu$ and standard deviation $\sigma$.
The treatment effect is computed as $\beta^x \gamma$.

As one comparison, we consider a GCM where $\xb$ is assumed to be just another covariate, rather than accounting for the correct causal structure (GM). We use the model,
\begin{align}
	\beta^a &\sim \mathrm{Normal}(0, 100)\\
\beta^x &\sim \mathrm{Normal}(0, 100)\\
	y_\omega &\sim \mathrm{Normal}(\beta^{a} a_\omega + \beta^x x_\omega + \epsilon_{\omega_1\diamond}^y + \epsilon_{\diamond\omega_2}^y, 0.1) \quad \quad \quad \quad \epsilon_{\omega_1\diamond}^y, \epsilon_{\diamond\omega_2}^y \sim \mathrm{Normal}(0, 100),
\end{align}
which is the same as the model of $\yb$ in the standard Bayesian GCM above, but now instead we report the treatment effect as $\beta^a$.

As another comparison, we consider a conventional causal model that has the correct causal graph but does not account for correlation across rows or columns (CM).
\begin{align}
	\gamma &\sim \mathrm{Normal}(0, 100)\\
	x_\omega &\sim \mathrm{Normal}(\gamma a_\omega, 0.1) \\
\beta^a &\sim \mathrm{Normal}(0, 100)\\
\beta^x &\sim \mathrm{Normal}(0, 100)\\
	y_\omega &\sim \mathrm{Normal}(\beta^{a} a_\omega + \beta^x x_\omega, 0.1) 
\end{align}
The treatment effect is computed as $\beta^x \gamma$.

\paragraph{Bayesian inference.} We implement the model in NumPyro \citep{Phan2019-br}. We perform stochastic variational inference using a diagonal Gaussian mean field approximation to the posterior, using NumPyro's \verb|AutoNormal| method to construct the variational family \citep{Kucukelbir2017-ir}. 
We optimize the variational approximation using Adam, taking 100{,}000 steps of size 0.0005 \citep{Kingma2015-ej}.
We use Monte Carlo to compute the posterior mean and credible intervals of the treatment effect $\beta^x \gamma$ for the GCM and CM models, drawing 10{,}000 samples from the approximate posterior over $\beta^x$ and $\gamma$.
For the GM model, we calculate analytically the posterior mean and credible interval of the treatment effect $\beta^a$ from the Gaussian variational approximation.   

\paragraph{Simulation configuration.} We create 10 independent datasets from the data generating process, each with 200 rows and 50 columns.
For each, we also create a smaller dataset where we subset the array to 100 rows and 25 columns.
We run inference on each of these 20 datasets and report the average error of the posterior mean treatment effect (\Cref{fig:array_gcm_avg_error}) and the average calibration over the 20 datasets, i.e. how often the true treatment effect fell in a 95\% credible interval (\Cref{fig:array_gcm_coverage}). The credible interval is taken to be between the 2.5th and 97.5th percentiles of the posterior.
	
\section{Robinson decomposition} \label{apx:robinson}

We derive the Robinson decomposition for the genomic GCM \citep{Nie2021-ae}.
Write the effect as a vector with an entry for each possible nucleotide variant, $\tau(\ab_{-0})_a = \textsc{CATE}_{\wt \to a}(\ab_{-0})$, and similarly represent the treatment $a_0$ as a vector, a one-hot encoding of the nucleotide. Then,
\begin{align}
	y_0 = m(\ab_{-0}) + &(a_0 - e(\ab_{-0})) \cdot \tau(\ab_{-0}) + \epsilon_0 \label{eqn:robinson}\\
	\mathrm{where\,\,  } & m(\ab_{-0}) \triangleq \Eb[Y_0 \mid \ab_{-0}]\\
	&  e(\ab_{-0})_a \triangleq \pr(a_0 = a \mid \ab_{-0})\\
	& \epsilon_0 \triangleq y_0 - (\Eb[Y_0 \mid a_0 = a_{\wt}, \ab_{-0}] + a_0 \cdot \tau(\ab_{-0}))
\end{align}
Here $m(\cdot)$ is the conditional mean outcome, $e(\cdot)$ is the treatment propensity, and $\epsilon_0$ is a random noise variable describing the residual. Crucially, this random noise has conditional mean zero, $\Eb[\epsilon_0 \mid \ab] = 0$. So, by fitting the model in \Cref{eqn:r-propensity,eqn:r-outcome} we can learn the CATE.

\section{Genomic Application}
Each neural network model consists of a linear convolution, a softplus nonlinearity, and a linear layer. 
The convolutional filter has a size of 3 nucleotides, and outputs a single feature. The linear layer has output dimension 1 for the outcome model $\mu_\theta$ and the conditional mean outcome $m_\phi$, and output dimension 3 for the treatment effect $\tau_\theta$. For the propensity model, the output dimension is 4 and we include a final softmax output to produce values on the simplex over nucleotides.

We parameterize reverse complement equivariant outcome models by conjoining, averaging model's output with its output for the reverse complement.
For the propensity model, which produces nucleotide-specific predictions, we instead average with the \textit{reverse complement} of the output applied to the window's reverse complement, since this corresponds to the propensity of the variant at the complementary DNA strand.
We did not enforce reverse complement equivariance on the treatment effect model $\tau_\theta$: the treatment model outputs predictions about the difference between two nucleotides, so enforcing reverse complement equivariance requires combining multiple outputs, potentially introducing additional regularization bias. Preliminary experiments showed this combination degraded performance.

We place $\mathrm{Laplace}(0, 1)$ priors on each dimension of the weight parameter of the convolutional filter and of the linear layer, and $\mathrm{Normal}(0, 100)$ priors on each dimension of the bias parameter in the filter and linear layer.
We use the same $\mathrm{Laplace}(0, 1)$ and $\mathrm{Normal}(0, 100)$ priors on the conventional linear model (CM). 
Each model describes the mean $y_\omega$ as normally distributed with standard deviation $\sigma=0.1$.

We perform stochastic variational inference using a diagonal Gaussian mean field approximation to the posterior, using NumPyro's \verb|AutoNormal| \citep{Phan2019-br,Kucukelbir2017-ir}. 
We optimize the variational approximation using Adam, taking 150{,}000 steps of size 0.005 \citep{Kingma2015-ej}.
We do not update the propensity model based on the outcome data $\yb$, using \verb|stop_gradient| to exclude gradient information during training, following prior work on Bayesian propensity models \citep{Zigler2016-ix,Jacob2017-hu,Hahn2018-uu}.

We use Monte Carlo to compute the posterior mean of the treatment effect, drawing 10 samples from the approximate posterior at every position $\omega$ in the sequence $\ab$.
We create 20 independent $\yb$ datasets from each data generating process (\Cref{eqn:genomic-test-f,eqn:genomic-test-f-2}), 10 with noise standard deviation $\sigma=0.1$ and 10 with $\sigma=0.3$.

\end{document}